\documentclass[journal]{IEEEtran}
\usepackage[style=ieee,maxnames=4,minnames=3,maxbibnames=3]{biblatex}

\usepackage{amsfonts}
\usepackage[cmex10]{amsmath}
\usepackage{amsthm}
\DeclareMathOperator*{\argmax}{arg\,max}

\theoremstyle{definition}

\theoremstyle{remark}

\theoremstyle{remark}

\theoremstyle{remark}
\newtheorem{lemma}{Lemma}
\newtheorem{theorem}{Theorem}
\theoremstyle{remark}

\usepackage{array}
\usepackage{mdwmath}
\usepackage{mdwtab}
\usepackage{eqparbox}
\usepackage[caption=1]{caption}
\usepackage{stfloats}
\usepackage{url}
\usepackage{amssymb}
\usepackage{float}
\usepackage{indentfirst}
\usepackage{makeidx}
\usepackage{tabularx}
\usepackage{cases}
\usepackage{mathtools}
\usepackage{color}
\usepackage{bm}
\usepackage{stfloats}
\usepackage{lipsum}
\usepackage{amsfonts}
\usepackage{bm}
\usepackage{dsfont}
\usepackage{lipsum}
\usepackage{graphicx}
\usepackage{url}
\usepackage{bbm}
\newcommand{\RNum}[1]{\uppercase\expandafter{\romannumeral #1\relax}}
\ifCLASSOPTIONcompsoc
\usepackage[caption=false, font=normalsize, labelfont=sf, textfont=sf]{subfig}
\else
\usepackage[caption=false, font=footnotesize]{subfig}
\fi
\usepackage{multirow}
\makeatletter

\newif\if@restonecol
\makeatother

\usepackage[linesnumbered,ruled,vlined]{algorithm2e}
\usepackage{algpseudocode}
\usepackage{amsmath}

\makeatother

\usepackage{amsthm}
\theoremstyle{definition} 
\theoremstyle{remark} 
\begin{document}


\title{{Over-the-Air Split Machine Learning in Wireless MIMO Networks}}

\author{
        Yuzhi~Yang,~
        Zhaoyang~Zhang,~
        Yuqing~Tian,~\\
        Zhaohui~Yang,~
        Chongwen~Huang,~
        Caijun~Zhong,~
        and Kai-Kit~Wong\\
\thanks{This work was supported in part by National Key R\&D Program of China under Grant 2020YFB1807101 and 2018YFB1801104, and National Natural Science Foundation of China under Grant U20A20158, and 61725104.}
\thanks{Y.~Yang, Z.~Zhang (Corresponding Author), Y.~Tian,  Z.~Yang, C.~Huang, and  C.~Zhong are with the College of Information Science and Electronic Engineering, Zhejiang University, Hangzhou 310007, China, and with the International Joint Innovation Center, Zhejiang University, Haining 314400, China, and also with Zhejiang Provincial Key Lab of Information Processing, Communication and Networking (IPCAN), Hangzhou 310007, China. Z. Yang is also with Zhejiang Lab, Hangzhou, 311121, China. (e-mails: \{yuzhi\_yang, ning\_ming, tianyq, yang\_zhaohui, chongwenhuang, caijunzhong\}@zju.edu.cn)}
\thanks{K.-K.~Wong is with the Department of Electronic and Electrical Engineering, University College London, WC1E 6BT London, UK, and also with School of Integrated Technology, Yonsei University,
Seoul, 03722, Korea. (email:  kai-kit.wong@ucl.ac.uk)}
}

\maketitle
\begin{abstract}
In split machine learning (ML), different partitions of a neural network (NN) are executed by different computing nodes, requiring a large amount of communication cost. As over-the-air computation (OAC) can efficiently implement all or part of the computation at the same time of communication, thus by substituting the wireless transmission in the traditional split ML framework with OAC, the communication load can be eased. In this paper, we propose to deploy split ML in a wireless multiple-input multiple-output (MIMO) communication network utilizing the intricate interplay between MIMO-based OAC and NN. The basic procedure of the OAC split ML system is first provided, and we show that the inter-layer connection in a NN of any size can be mathematically decomposed into a set of linear precoding and combining transformations over a MIMO channel carrying out multi-stream analog communication. The precoding and combining matrices which are regarded as trainable parameters, and the MIMO channel matrix, which are regarded as unknown (implicit) parameters, jointly serve as a fully connected layer of the NN. Most interestingly, the channel estimation procedure can be eliminated by exploiting the MIMO channel reciprocity of the forward and backward propagation, thus greatly saving the system costs and/or further improving its overall efficiency. The generalization of the proposed scheme to the conventional NNs is also introduced, i.e., the widely used convolutional neural networks. We demonstrate its effectiveness under both the static and quasi-static memory channel conditions with comprehensive simulations.
\end{abstract}


\section{Introduction}
\subsection{Motivation}
In future sixth-generation (6G) wireless communication systems, human-like intelligence will be brought everywhere \cite{toward6G}. The rapid development of artificial intelligence leads to booming mobile machine learning (ML) applications and requires vast data interaction. The integration of ML and wireless communications leads to an emerging area called split ML, which distributes a neural network (NN) to several edge devices, thus reducing the system computation burden.
In a split ML system, each device first proceeds forward computation on the allocated partial NN and then transmits calculated intermediate results to the next device for further computation.
The backward propagation process is conducted in a way similar to forward transmission but in a backward order.
Through cooperative computation, split ML can be applied to solve deep learning-based joint source-channel coding (JSCC) problem \cite{DeepJSCC} and other ML problems in cloud networks and the Internet of things.
However, split ML requires frequent information exchange among devices, leveraging a heavy communication burden on wireless networks.
Thus, deploying split ML over wireless networks calls for the design of new wireless techniques based on a communication-and-computation integration approach.

Recently, over-the-air computation (OAC) has emerged as a meaningful approach to innovate the traditional wireless digital communication framework \cite{OAC}. Through using the intrinsic linear superposition property of wireless channels, OAC enables wireless communication systems to compute some large-scale, but simple calculation tasks \cite{OACfunc}.
Previous work \cite{FLOAC} shows that if a group of devices simultaneously transmits analog modulated signals, the receiver obtains an aggregated signal, which can be directly applied to the typical federated learning (FL) network. However, traditional OAC work only considers the weighted sum of the edge users' messages within a time-synchronized multiple-access framework and has not been well suited to the split ML system.
In a typical wireless system with multiple-input multiple-output (MIMO) antenna arrays, a group of antennas at the transmitter sends different signals simultaneously, while another group of antennas at the receiver collects the transmitted signal. On the receiver side, each antenna can individually receive the aggregated signal from multiple antennas of the transmitter. Intuitively, the MIMO channel can be viewed as a multiplication-and-addition procedure on the transmitted analog signals, which is widely used in NNs, and thus is of potential to be applied to the split ML in general wireless systems.

There is an intricate interplay between MIMO-based OAC and NN.
A MIMO channel can provide a full connection between the inputs and the outputs and can thus be viewed as a \textit{weighted sum calculator}.
By different channel gains, each antenna on the receiver side conducts calculations with different equivalent weights.
Unlike fully connected layers in NNs, which can be viewed as \textit{controlled} weighted sums, the equivalent weights in MIMO systems are determined by the channel matrices, which are determined by the environment and hence \textit{uncontrollable}.
To control the equivalent weight in such systems, the typical precoding and combining operations in a MIMO system can be properly exploited.
Both procedures are controllable linear transformations on the signals and are inherent in MIMO systems.
As a result, we can control the parameters in the equivalent weighted sum of the overall system by controlling the precoding and combining matrices.

However, implementing the MIMO OAC in split ML still faces two fundamental issues: the forward and backward channels are different and may not be accurately known, and the analog transmission in OAC results in unavoidable noise. To deal with the uncertainty issue of the MIMO channel, we find that  the forward-backward propagation of a NN and the channel reciprocity of a wireless channel are mathematically related (see Section \ref{sec::2}-B for details). This can be exploited to deploy a NN through MIMO-based OAC, which can still lead to correct gradients even without any prior knowledge about the MIMO channel as long as the channel reciprocity and quasi-stability are assumed. Moreover, most NNs can work well after proper training, even when the intermediate results cannot be fully and accurately interpreted. Such unexplainably in NNs inspires us that casting a deterministic linear transformation, i.e., multiplying an implicit matrix determined by the MIMO channel on the intermediate results of a NN through the whole training and test process, may not intensively deteriorate the performance.

The other issue about the noise in wireless communication is not fatal in NNs. From the perspective of information theory, both digital and analog communication can be optimal in wireless communications such as sensor networks \cite{optimal}. However, the digital communication system can reduce transmission errors by employing error-detecting and correcting codes, whereas transmission errors can only be restricted but never eliminated in analog communication. Moreover, it is found that noise is tolerable and sometimes even becomes a training trick in NNs \cite{noisy}, which can also be viewed as implicit dropout \cite{dropout}.
Since the devices transmit unexplainable, intermediate results of the NN in split ML systems, the advantage of progressive error-free is insignificant in such applications. For instance, Jankowski et al. \cite{DeepJSCC} show that analog communication performs better than digital communication in the deep learning-based JSCC problem, which is a particular case of split ML.
\subsection{Related Works}
Split ML is a method where multiple computation nodes cooperatively execute an ML application. In such a system, the ML model is split into multiple parts allocated to different computation nodes. Each node executes a part of the ML model in order and transmits the intermediate results to the next one. When training the NN, the nodes also execute backward propagation in backward order. Most of the existing works \cite{kang17, hu19, hivemind} on split ML focus on how to distribute the model on the nodes in a way to minimize the total communication and computation delay.
There are also some other works bringing up specialized NN structures for split ML \cite{bottlenet, bottlenet++, JALAD}.
Recent work also tries to deploy a proper NN architecture on a given communication network \cite{jmsnas}. Their framework utilizes the neural architecture search method to meet latency and accurate requirements.
However, the above works \cite{kang17, hu19, hivemind, bottlenet, bottlenet++, JALAD, jmsnas} all consider ideal communication among the computing nodes, which ignores the communication scheme design.

On the other hand, in communication systems, OAC is usually deployed in multiple access systems to compute the weighted sum or some easy mathematical operations such as geometric mean, polynomial, and Euclidean norm \cite{OACfunc}. Hence most OAC works are mainly used in FL, where the weighted sum is widely deployed. For example, the authors in \cite{FLOAC} optimize the number of simultaneous accesses, which may improve the efficiency of FL. Besides, Zhu et al. apply broadband analog aggregation to improve OAC in FL with multiple bands \cite{Zhu20}, while Shao et al. consider FL with misaligned OAC \cite{misaligned}.

OAC for multiple access systems still has significant drawbacks.
Firstly, OAC requires strict synchronization among all transmitters, which is hard to realize. Moreover, OAC does not support backward communication, which is rarely considered in previous works as far as we know.
Furthermore, the above works \cite{OACfunc, FLOAC, Zhu20, misaligned} do not consider the MIMO system, which is widely used in practical scenarios.
The multiple antennas of the transmitter in a MIMO system can be viewed as a group of transmitters in the multiple access scheme, which overcomes the drawback of synchronization.
The authors of \cite{MIMOOAC} apply MIMO OAC to multimodal sensing. However, in \cite{MIMOOAC}, the scenario is still a multiple access system where all channels are MIMO channels, and the task of OAC is still the weighted sum, which can be regarded as a direct expansion of previous designs in \cite{OACfunc, FLOAC, Zhu20, misaligned}.
Besides, the implementation of \cite{MIMOOAC} is also strictly limited to FL applications, which is not suitable for split ML.

In other fields, OAC has also provided an alternative to traditional NNs by realizing parts of NNs with acoustic \cite{OAC_acoustic}, optical \cite{OAC_opt}, and radio frequency \cite{OAC_RF} signals.
The authors in \cite{OAC_acoustic, OAC_opt, OAC_RF} use the characteristics of the target systems similar to NNs, and employ the target systems as part of a NN.
OAC systems calculate aggregated results of multiple inputs from different transmitters or time slots by moderating the environments or some parts of the system.
Recently, in wireless communications, Sanchez et al. also realize NNs with the help of multiple paths, and reconfigurable intelligent surfaces \cite{AirNN}. Their system transmits the intermediate output of NNs via time-sequential signals and uses the delay of multiple paths to realize one-dimensional convolution. However, in this paper, we use MIMO channels to realize fully connected layers through multiplexed signals.
\subsection{Contributions}
The main contributions of this paper are summarized as follows:
\begin{itemize}
\item A split ML framework is proposed for wireless MIMO networks by exploiting the MIMO's OAC capability, which not only enables high-throughput and efficient wireless transmission but also reduces the overall computation load by synergistically incorporating the split ML process with the wireless transmission procedure rather than just taking it as a bit pipe.

\item We show that the inter-layer connection in a NN of any size can be mathematically decomposed into a set of linear precoding and combining transformations over the MIMO channels. Therefore, the precoding matrix at the transmitter and the combining matrix at the receiver of each MIMO link, as well as the channel matrix itself, can jointly serve as a fully connected layer of a NN.

\item By exploiting the reciprocity of the MIMO channel in the forward and backward propagation procedures in the proposed framework, we find it unnecessary to conduct explicit channel estimation as otherwise indispensable in conventional communication systems, thus further improving the overall communication and computation efficiency.

\item We also provide some design rules for the proposed system so as to apply it to a fully connected layer of any size in a fully connected NN or a convolutional layer of any size in a convolutional NN. Simulation results show that the proposed scheme is efficient under both static and slowly-varying memory channel conditions.
\end{itemize}

The remainder of the paper is organized as follows. We first introduce the proposed system in Section \ref{sec::2}.
We then mathematically provide some principles and propose a training algorithm in Section \ref{sec::3}.
We extend the proposed system to convolutional NNs and compare different implementations of the system in Section \ref{sec::4}. Numerical results are provided in Section \ref{sec::5}. Section \ref{sec::6} concludes the paper and provides future directions.

\subsection{Notations}
In this paper, we use bold italic lower-case letters for vectors and bold letters for matrices. All the vectors and matrices are assumed to be complex.
The meanings of frequently used notations are summarized in detail in Table \ref{table::0}.
\begin{table}[!htp]\caption{Notations used in this paper}
\label{table::0}
\footnotesize\begin{tabular}{c|c}
\hline
Notation & Meaning \\\hline
$N_t, N_r$ & The number of antennas on the transmitter and the receiver.\\\hline
$N_i, N_o$ & The input and output sizes of a NN layer.\\\hline
$\bm{x}, \bm{y}$ & The signals before precoding, and after combining,\\
& also for the input and output of a NN layer.\\\hline
$\bm{x}_k, \bm{y}_k$ & The transmitted signals after precoding\\
& and the received signal before combining,\\\hline
$\mathbf{H}$ & The channel matrix.\\\hline
$\mathbf{P}_k, \mathbf{C}_k$ & The precoding and combining matrices of transmission $k$.\\\hline
$\bm{n}$ & The noise vector.\\\hline
$r$ & The roughly estimated rank of $\mathbf{H}$.\\\hline
$\mathbf{W}$ & A trainable matrix in a NN.\\\hline
$\mathbf{K}$ & A trainable convolutional kernel in a NN.\\\hline
$\bm{g}_{a}$ & The gradient corresponding to the subscript.\\\hline

\end{tabular}
\end{table}
\section{Basic OAC Unit}\label{sec::2}
In this section, we first briefly introduce the system model and then propose a MIMO OAC-based approach to accelerate communication in split ML.

\subsection{System Settings}
\begin{figure*}
    \centering
       \includegraphics[width=0.7\linewidth]{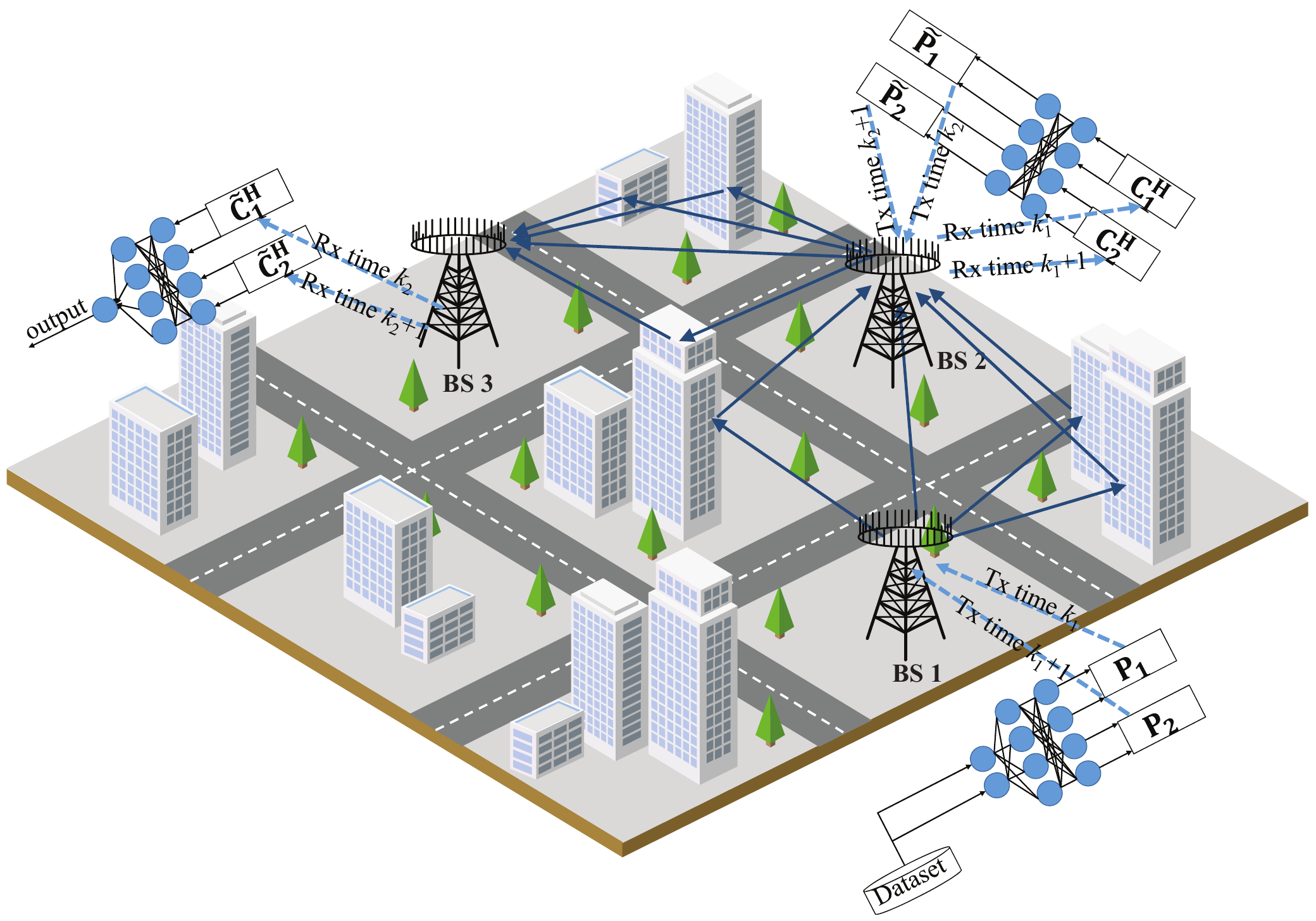}
	  \caption{The considered split ML system based on MIMO-based OAC scheme.}
	  \label{system}
\end{figure*}
Consider a MIMO OAC-based split ML system with multiple base stations, as shown in Fig.~\ref{system}.
Each base station is equipped with multiple antennas, and the MIMO channels among base stations are quasi-stable. A deep NN is split into several segmented NN parts, each of which is deployed on one specific base station. The base stations process the forward computation and backward propagation of the assigned NN fragment.
For simplicity, unless otherwise stated, we consider the split ML system with only one splitting point since it can be easily generalized to multiple-split conditions by conducting a set of single-split systems.
To simplify the description, we use ``transmitter'' and ``receiver'' to refer to the transmitter and receiver with $N_t$ and $N_r$ antennas in the forward transmission, respectively.

Assuming that the channel between the transmitter and the receiver is a quasi-stable MIMO channel with reciprocity, i.e., the forward channel is $\mathbf{H}\in\mathcal{C}^{N_r\times N_t}$, and the backward channel is $\mathbf{H}^T$,
we do not make any further priori hypotheses of the channel as the split ML applications may be deployed in different scenarios.
Under the system setting, we assume that the rank of channel $\mathbf{H}$ is known to be $r$, which determines the maximum amount of dataflows transmitted simultaneously under $\mathbf{H}$.
We note that the rank $r$ is mainly determined by the number of paths, which can be approximately known from the channel model and the number of scatterers.
Therefore, although obtaining the accurate value of $r$ is impossible without channel estimation, we can get $r$ roughly.
Some small singular values below a certain threshold can be regarded as zeros in the aspect of engineering.
Hence, the rank $r$ can be underestimated in some cases.
When designing the system, we assume that $r$ is known exactly, and we later show the cases where there exists a mismatch numerically in Section \ref{sec::5}.
Besides, to better realize OAC, we may apply other techniques, such as applying the orthogonal frequency division multiplexing technique to realize multiple transmissions simultaneously. However, we do not consider such techniques since they do not affect the principles we use in the paper and can be easily combined with the proposed system.
\subsection{OAC for Split ML}
We consider a two-node split ML system, where the original NN is split into two parts, corresponding to the first several layers and the other layers, which are respectively deployed on the transmitter and the receiver. A fully connected layer between the transmitter and the receiver is realized through the OAC technique.
We consider complex NNs \cite{complexResNet}, which have shown effectiveness in graph classification tasks in the proposed system. Since most data in wireless communication is complex, complex NNs are also potential in native wireless communication learning tasks. In complex NNs, the parameters, gradients, as well as intermediate results are complex numbers, and the backpropagation is almost the same as that of traditional real NNs.
The transmitter can be regarded as a federation of several synchronized single-antenna transmitters in MIMO systems. Hence OAC still works similarly with multiple access systems.
We will later discuss how to extend the proposed system to convolutional layers in Section \ref{sec::CNN}.
In particular, a fully connected layer in a NN can be represented by
\begin{equation}
\bm{x}_{\text{next}}=\phi(\mathbf{W}\bm{x}+\bm{b}),
\end{equation}
where $\bm{x}\in\mathcal{C}^{N_i * 1}$ and $\bm{x}_{\text{next}}$ is the input and the output of the considered layer, respectively, $\mathbf{W}\in\mathcal{C}^{N_o * N_i}$ and $\bm{b}\in\mathcal{C}^{N_o * 1}$ are the network parameters, and $\phi(\cdot)$ is the element-wise activation function.

Considering the procedure of OAC, the transmitted and received signals may be amplified. Rather than using amplified parameters, we adopt a batch normalization (BN) layer to remain other parameters in a reasonable range and accelerate training \cite{BN}. A BN layer casts an element-wise linear transformation to the forward propagated data and normalizes them into a predetermined mean and variation distribution.
The BN layer normalizes the intermediate output calculated by OAC to a reasonable range and can counteract the widely existing amplifications in communication systems, making it suitable for the considered system.
With the BN layer, the output becomes
\begin{equation}
\bm{x}_{\text{next}}=\phi(\text{BN}(\bm{y}+\bm{b}))=\phi(\text{BN}(\mathbf{W}\bm{x}+\bm{b})),\label{forward_NN}
\end{equation}
where $\bm{y}=\mathbf{W}\bm{x}$ is calculated through OAC and $\bm{x}_{\text{next}}$ is calculated on the receiver by the received $\bm{y}$. In the remaining parts of this paper, we only consider the OAC part of the system, i.e., how to calculate $\bm{y}$ from $\bm{x}$, since other parts have been well studied in traditional NNs. We discuss the mathematical similarity between the channel reciprocity and the forward-backward propagation below, which is also shown in Fig. \ref{fig::2}.

\begin{figure*}[t]
    \centering
    \includegraphics[width=0.8\linewidth]{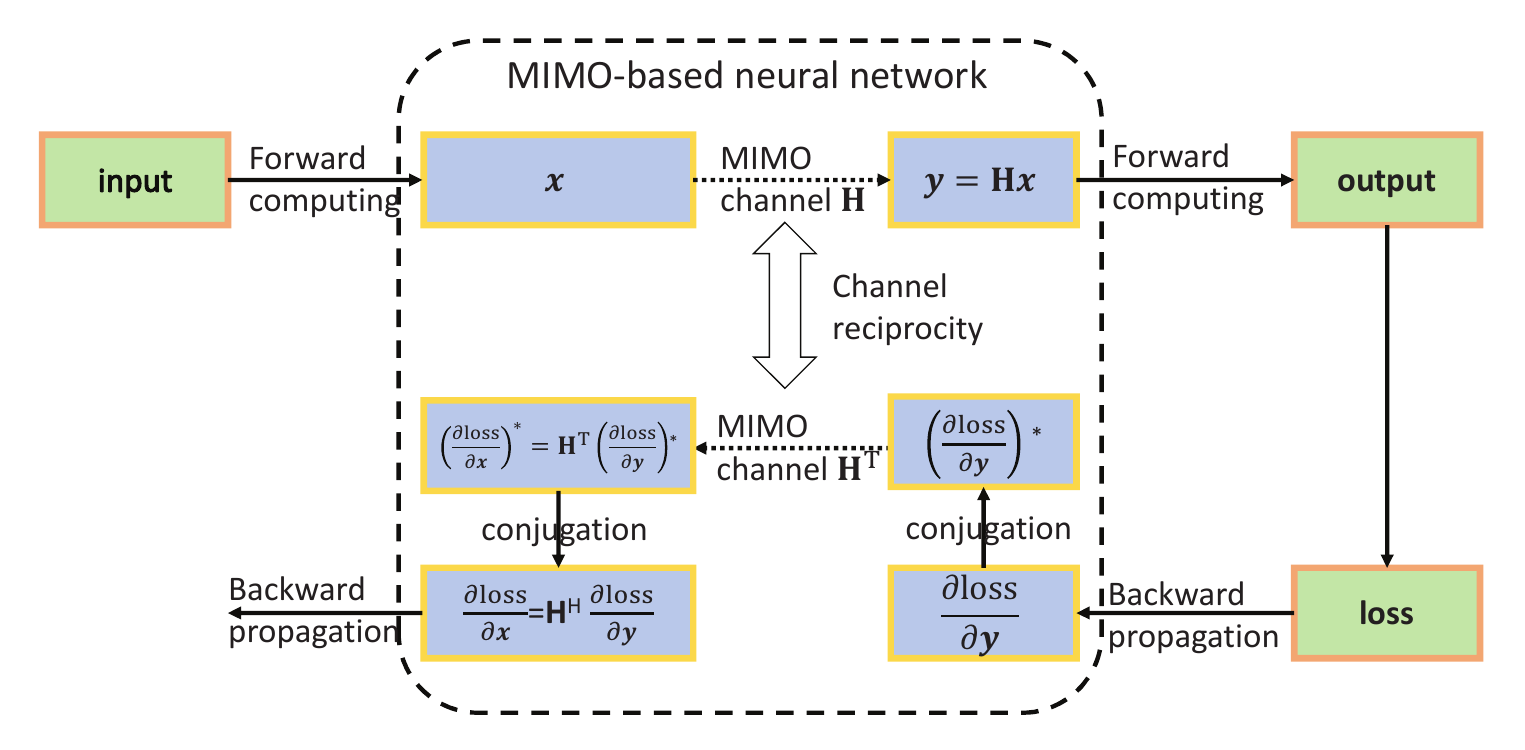}
	\caption{The relationship between channel reciprocity and backward propagation}
	\label{fig::2}
\end{figure*}

\subsubsection{Forward Computation}
For each intermediate result, since its size is usually far larger than the maximum number of simultaneous transfers in a single OAC, we need to conduct $K$ times of OAC for each forward computing, where $K$ is the minimum number of times of progressing the intermediate results and its value is discussed in the following. In the $k$-th OAC, the precoding and combining matrices are denoted by $\mathbf{P}_k \in\mathcal{C}^{r*N_i}$ and $\mathbf{C}_k \in\mathcal{C}^{r*N_o}$, respectively. Hence, the precoding, transmitting, and combining procedures of the $k$-th OAC can be described respectively as
\begin{eqnarray}
\bm{x}_{t,k}&=&\mathbf{P}_k\bm{x},\label{eq3}\\
\bm{y}_{r,k}&=&\mathbf{H}\bm{x}_{t,k}+\bm{n}_k,\label{eq4}\\
\bm{y}_k&=&\mathbf{C}^{H}_k\bm{y}_{r,k},\label{eq5}
\end{eqnarray}
where $\bm{n}_k\sim\mathcal{CN}(0,\sigma^2 \mathbf{I})$ is the Gaussian white noise vector of the forward transmission with the variance $\sigma^2$. The receiver then combines all the $K$ received signals, and the output of OAC becomes
\begin{equation}
\begin{aligned}
\bm{y} &= \sum_{k=1}^K \bm{y}_k= \sum_{k=1}^K \mathbf{C}^{H}_k\mathbf{H}\mathbf{P}_k \bm{x}+\bm{n}_k.\label{forward_air}
\end{aligned}
\end{equation}
\subsubsection{Backward Propagation}
Here, we consider the gradient backward propagation in our system. Let the gradient of $\bm{y}$ be $\bm{g}_y\triangleq\partial\textrm{loss}/\partial\bm{y}$. According to (\ref{eq3})-(\ref{eq5}) and the principle of gradient backward propagation, we can easily obtain the following equations:
\begin{eqnarray}
\bm{g}_{C_k}&=&\bm{y}_{t,k}\bm{g}_y^H,\\
\bm{g}_{x_{t,k}}&=&\mathbf{H}^{H}\mathbf{C}_k\bm{g}_y,\label{back_air}\\
\bm{g}_{P_k}&=&\bm{x}\bm{g}_{x_{t,k}}^H,\\
\bm{g}_x&=&\sum_{k=1}^K \mathbf{P}_k^{H}\bm{g}_{x_{t,k}}.
\end{eqnarray}

All the above equations can be easily calculated by either the receiver or the transmitter except for \eqref{back_air}. Equation \eqref{back_air} can be calculated with noise through backward OAC, where the receiver acts as a transmitter in OAC and transmits $\bm{g}_y^*$ with precoding matrix $\mathbf{C}_k^*$. For the backward propagation, the received signal at the transmitter becomes
\begin{equation}
\tilde{\bm{g}}_{x_{t,k}}=\mathbf{H}^T\mathbf{C}_k^*\bm{g}_y^* + \tilde{\bm{n}}_k=\bm{g}_{x_{t,k}}^*+ \tilde{\bm{n}}_k,\label{back_air_new}
\end{equation}
where $\tilde{\bm{n}}_k\sim\mathcal{CN}(0,\sigma^2 \mathbf{I})$ is the Gaussian white noise vector in the backward transmission. Equation \eqref{back_air_new} reveals that channel reciprocity ensures that the backward OAC is identical to backward propagation in NNs if we transmit the conjugation of the gradient and ignore the noise.
\section{System Implementation and Algorithm Design}\label{sec::3}
In this section, we first provide theoretical findings for split ML based on fully connected layers over wireless networks.
We then formulate the optimization problem and propose the training algorithm for the system.

\subsection{Theoretical Findings in the Proposed System}
Here, we only consider the matrix multiplication in the fully connected layer, i.e., $\bm{y}=\mathbf{W}\bm{x}$.
To ensure that the split ML system works well, we must guarantee that the OAC system can be equivalent to any parameter $\mathbf{W}$ of the fully connected layer. Hence, we have the following feasible condition for noiseless channels.

To realize the function of a fully connected layer, condition \eqref{proposition} should have at least a feasible solution $(\mathbf{P}_k,\mathbf{C}_k), k=1,\cdots,K$ for any channel $\mathbf{H}\in\mathcal{C}^{N_r*N_t}$ with rank $r$ and parameter $\mathbf{W}\in\mathcal{C}^{N_o*N_i}$, where $N_i$ and $N_o$ are the dimensions of the input and the output, respectively.
\begin{equation}\label{proposition}
\mathbf{W} = \sum_{k=1}^K \mathbf{C}^{H}_k\mathbf{H}\mathbf{P}_k.
\end{equation}

From feasible condition \eqref{proposition}, we can obtain the following lemma.
\begin{lemma}\label{the::1}
If the rank of $\mathbf{G}\triangleq (\mathbf{C}^{H}_1\mathbf{H},\cdots,\mathbf{C}^{H}_K\mathbf{H})$ is at least  $N_o$ and $\mathbf{P}_k$ can take any values from $\mathcal{C}^{N_t\times N_i}$, then feasible condition \eqref{proposition} always holds. Symmetrically, if the rank of $\mathbf{G}'\triangleq (\mathbf{P}_1^H\mathbf{H}^{H},\cdots,\mathbf{P}_K^H\mathbf{H}^{H})^H$ is at least $N_i$ and $\mathbf{C}_k$ can take any values from $\mathcal{C}^{N_r\times N_o}$, then feasible condition \eqref{proposition} always holds.
\end{lemma}
\begin{proof}
We first prove the first part. Since $\mathbf{G}\triangleq (\mathbf{C}^{H}_1\mathbf{H},\cdots,\mathbf{C}^{H}_K\mathbf{H})$, we have $\sum_{k=1}^K \mathbf{C}^{H}_k\mathbf{H}\mathbf{P}_k=\mathbf{G}(\mathbf{P}^{H}_1, \cdots, \mathbf{P}^{H}_K)^H$. Due to the fact that the rank of $\mathbf{G}$ is at least $N_o$, we can easily observe that for arbitrary $\mathbf{W}\in \mathcal{C}^{N_o\times N_i}$, let $\mathbf{P}_i, k=1, \cdots, K$ be the submatrix consisting the $(kr-k+1)$-th to the $(kr)$-th rows of $\mathbf{G}^{-1}\mathbf{W}$. Then, $\mathbf{W}=\sum_{k=1}^K \mathbf{C}^{H}_k\mathbf{H}\mathbf{P}_k$, which completes the first part of Lemma 1. The second part of Lemma \ref{the::1} is symmetric to the first part, hence Lemma \ref{the::1} is proved.\end{proof}
Based on Lemma 1, we further obtain the requirement for the time $K$ of OAC and the rank $r$ of $\mathbf{H}$ that can satisfy\eqref{proposition}.
\begin{theorem}\label{the::2}
Under the condition that $\mathbf{P}_k$ and $\mathbf{C}_k$ can take any values from $\mathcal{C}^{N_t\times N_i}$ and $\mathcal{C}^{N_r\times N_o}$ respectively, feasible condition \eqref{proposition} holds if and only if $Kr\geq \min\{N_i,N_o\}$.
\end{theorem}
\begin{proof}
We first provide the necessary condition.
Since $\textrm{rank}(\mathbf{C}^{H}_k\mathbf{H}\mathbf{P}_k)\leq\textrm{rank}(\mathbf{H})=r$, we have
\begin{equation}
\textrm{rank}\left(\sum_{k=1}^K \mathbf{C}^{H}_k\mathbf{H}\mathbf{P}_k\right)\leq Kr.
\end{equation}
Hence if $Kr<\min\{N_i,N_o\}$, for any $\mathbf{W}$ with rank $\min\{N_i,N_o\}$, feasible condition \eqref{proposition} does not hold.

Then, we provide the sufficient condition.
We know $\textrm{rank}(\mathbf{I}_{K}\otimes\mathbf{H})=Kr$, where $\otimes$ denotes the Kronecker product and $\mathbf{I}_K$ denotes the $K\times K$ identity matrix. If $N_i\geq N_o$, we can easily know that there exist some $\mathbf{C}_k, k=1\cdots,K$, such that rank$(\mathbf{G})=Kr>N_o$, where $\mathbf{G}$ is defined in Lemma \ref{the::1}. From the first part of Lemma \ref{the::1}, feasible condition \eqref{proposition} holds. Similarly, if $N_i< N_o$, we can still find some $\mathbf{P}_k, k=1\cdots, K$, which satisfy the condition of the second part of Lemma \ref{the::1}, and hence feasible condition \eqref{proposition} also holds.

As a result, Theorem \ref{the::2} is proved.
\end{proof}
Lemma \ref{the::1} and Theorem \ref{the::2} give the sufficient and necessary conditions to feasible condition \eqref{proposition}.
Theorem 1 shows that we should at least conduct $\min\{N_i, N_o\}/r$ times of transmissions to satisfy feasible condition \eqref{proposition}. The proof to Theorem 1 also guides us that we must follow the conditions in Lemma \ref{the::1} to reach the bound.
We will later show how to design the system with minimum transmission times following the conditions above.

In practical communication systems, there exist noise in both forward and backward communications, and the maximum transmitting powers for both transmissions are limited. However, the influences of the noise on the forward and backward channels are not the same. In the forward channel, the noise directly takes part in the forward computation and results in worse performance on the task. However, the noise only affects the gradient calculation in the backward channel. The well-known stochastic gradient descent (SGD) \cite{SGD} algorithm uses gradients on random parts of the total dataset instead of the entire dataset and still works well since the expectation of the stochastic gradient equals the actual gradient.
Since the expectation of the gradient added with zero-mean noise still equals the real gradient, the noise in the backward channel may not destroy the learning performance.

Besides, we note that all NNs consist of only two components, i.e., linear layers and nonlinear activation functions.
Since the activation functions are always unary, the backward propagation procedures of both types of components are linear. Hence the backward propagation from the output of the NN to any parameter is linear. The only nonlinear part in the NNs training process may lie in the loss function at the output.
Hence if the noise at the backward receiving antenna is Gaussian, as shown in Fig. \ref{fig::system2}, compared with the corresponding centralized NN, the gradient for any parameter should be either noiseless or with Gaussian noise.

We note that the stochastic gradient Langevin dynamics (SGLD) \cite{SGLD1, SGLD2, SGLD3} also introduces noise to SGD. SLGD is proposed to accelerate the training of minibatch SGD, and the core inspiration of it is to introduce a carefully designed Gaussian noise to overcome the randomness brought by the sampling of minibatchs.
However, in the proposed system, the noise is brought by the wireless channel, whose distribution is hard to control.
As a result, the proposed system works similarly to the SLGD, but the origins of the noises are different.
Besides SLGD, we also give another way regarding this problem as follows.

Inspired by SGD, we know that the white noise in the backward channel does not have obvious significance in convergence, which can also be theoretically shown in the following Theorem \ref{the::3}.
Before introducing Theorem \ref{the::3}, we define the following notations.
Denote $h(\bm{x}; \bm{\theta})$ as the model, where $\bm{x}$ represents the input and $\bm{\theta}$ stands for the parameters.
With SGD, the loss of the $t$-th batch can be written by
\begin{equation}
f_{t}(\bm{\theta})\triangleq
\sum_{b=1}^{B} \ell\left(h\left(\mathbf{x}_{b,t}; \bm{\theta}\right), y_{b,t}\right),
\end{equation}
where $B$ represents the batch size, $\ell(\cdot,\cdot)$ denotes the loss function, $\mathbf{x}_{b,t}$ and $y_{b,t}$ stand for the $b$-th input data and the corresponding label of the $t$-th batch, respectively. Now, it is ready to provide the following theorem about convergence with noise in the backward channel.

\begin{theorem}\label{the::3}
If all loss functions $f_{t}(\bm{\theta}), t=1,2,\cdots$ are convex, all the parameters and the gradients are bounded, i.e., $\|f_t(\bm{\theta})-f_t(\bm{\theta}')\|\leq D$ and $\|\bm{g}_t\|\leq G$ for any $\bm{\theta}$ and $\bm{\theta}'$, where $\bm{g}_t=\nabla f_t(\bm{\theta})$, the SGD algorithm converges in $\mathcal{O}(\sigma^2 T^{-1/2})$ even with noise $\bm{n}$. However, the coefficient in the convergence time increases linearly with noise power.
\end{theorem}
\begin{proof}
Please refer to Appendix \ref{proof_the::3}.
\end{proof}

Theorem \ref{the::3} indicates that the introduction of Gaussian white noise in the backward propagation will not affect the order in the convergence rate of SGD.
Besides, the noise power only linearly changes the coefficient of the convergence rate.

Another problem comes from the amplification. In wireless communication systems, due to the effect of noise and the limited budget of the transmitting power, we must apply amplification for both forward and backward transmissions to match the transmit power, which may influence the performance of NNs. Hence, in the following, we discuss how to realize arbitrary amplification in the forward computation and backward propagation without scarifying the performance of NNs significantly.

In the forward computation, we can apply a simple power normalization before transmitting, i.e., transmit $\mathbf{P}_k\bm{x}/A$ instead of $\mathbf{P}_k\bm{x}$, where $A=\sqrt{\mathbb{E}(\bm{x}^H\mathbf{P}_k^{H}\mathbf{P}_k\bm{x})}$, as shown in Fig. \ref{fig::system2}. Since the normalization layer amplifies all the transmitted signals and there is a BN layer at the receiver to offset the effect of power amplification, it does not significantly affect the NN's result.

In backpropagation, normalizing the power of the gradient of the backward propagation only amplifies the gradient of the subsequent calculation on the same scale. To overcome this problem, we can use optimizers insensitive to the magnitude of the gradient, such as the Adam optimizer \cite{adam}.

\subsection{The Implementations for Fully Connected NNs}

\begin{figure*} [!htp]
    \centering
       \includegraphics[width=1\linewidth]{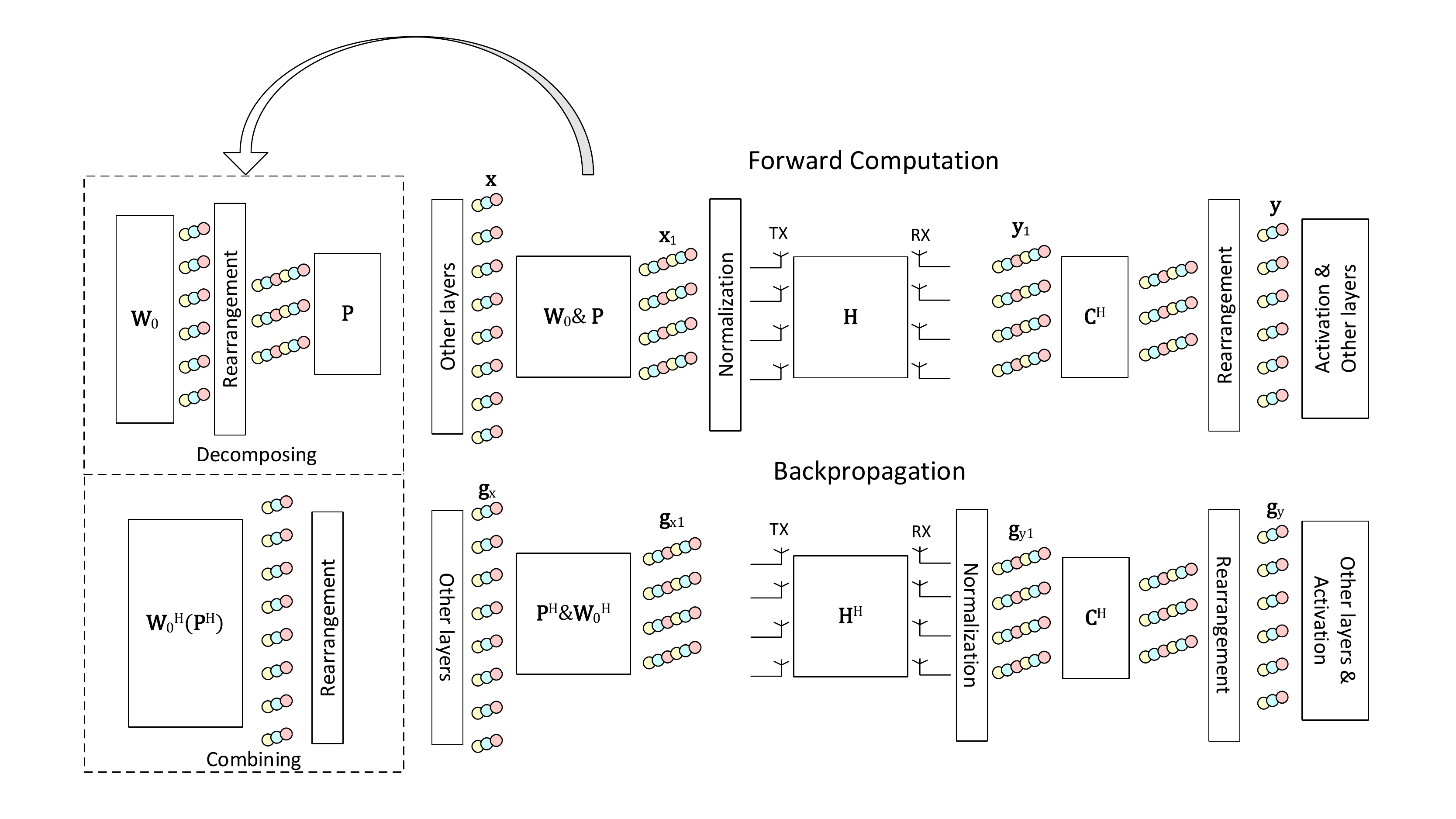}
	  \caption{Two transmitter parameterized implements of the proposed MIMO OAC-based split ML system. In this illustration, batch size $B=3$, the rank of the channel is $r=3$, the input and output sizes are $N_i =N_o =6$, and the numbers of antennas at the transmitter and the receiver are $N_t = N_r = 4$.}
	  \label{fig::system2}
\end{figure*}

Since the channel does not change over time, the optimum beamforming matrix remains the same among the transmissions.
Hence we can always construct a receiving combining matrix $\mathbf{C}\in\mathcal{C}^{N_r\times r}$ such that matrix $\mathbf{H}^{H}\mathbf{C}$ is of rank $r$.
As a result, the receiver can simultaneously decode $r$ independent data streams, i.e.,
\begin{equation}
\mathbf{C}_k=(\mathbf{0}_{N_r\times (k-1)r}, \mathbf{C}, \mathbf{0}_{N_r\times (K-k)r}), k=1,\cdots,K.\label{eq::C}
\end{equation}

According to Lemma \ref{the::1} and Theorem \ref{the::2}, the $\mathbf{C}_k$ above satisfies feasible condition \eqref{proposition} with the least number of transmissions $K$ when $N_i\geq N_o$. The procedure above is symmetric between $\mathbf{P}_k$ and $\mathbf{C}_k$ that we can also use the same precoding matrix $\mathbf{P}\in\mathcal{C}^{N_t\times r}$ such that matrix $\mathbf{H}\mathbf{P}$ is of rank $r$, and $\mathbf{P}_k=(\mathbf{0}_{N_t\times (k-1)r}, \mathbf{P}, \mathbf{0}_{N_t\times (K-k)r})$ for $k=1,\cdots,K$ to satisfy feasible condition \eqref{proposition} with the least $K$ when $N_i\leq N_o$. We refer to the former ones as receiver parameterized designs and the latter ones as transmitter parameterized designs, respectively. Based on symmetry, we can only discuss the transmitter parameterized designs in this section.

According to Lemma \ref{the::1}, we obtain a straightforward implementation by directly setting all the $\mathbf{P}_k$s to be trainable parameters as shown in Fig.~\ref{fig::system2}. However, the straightforward implementation does not take advantage of the fact that the structure of the precoding matrices should also be determined by the channel $\mathbf{H}$ to a certain degree. It is more desirable to separate the precoding matrices by two parts, i.e.,
\begin{equation}
\mathbf{P}_k=\mathbf{W}_0(\mathbf{0}_{N_r\times (k-1)r}, \mathbf{P}, \mathbf{0}_{N_r\times (K-k)r}),  k=1,\cdots,K,\label{eq::P}
\end{equation}
where $\mathbf{P}$ plays a similar role in the precoding matrices compared to $\mathbf{C}$ in the receiving combining matrices, and $\mathbf{W}_0$ is corresponding to the remaining NN parameters. The system is also illustrated in Fig. \ref{fig::system2}.
In this paper, as shown in Fig. \ref{fig::system2}, we refer to the former implementation as the combined design and the latter as the separated design, respectively. The detailed comparisons among the implementations are provided in Section \ref{sec::comp}.

\subsection{Problem Formulation}
We first analyze the transmitter parameterized and separated design for full-connected layers.
From equations \eqref{forward_air}, \eqref{eq::C}, and \eqref{eq::P}, the sum of the received signals of the forward transmissions $\bm{y}$ can be represented by
\begin{eqnarray}
&\bm{y}&=\left(\begin{array}{c}
\mathbf{C}^{H}\mathbf{H}\mathbf{P}\bar{\bm{x}}_1
+\mathbf{C}^{H}\bm{n}_1 \\
\vdots \\
\mathbf{C}^{H}\mathbf{H}\mathbf{P}\bar{\bm{x}}_K+\mathbf{C}^{H}\bm{n}_K
\end{array}\right),\\
&\left(\begin{array}{c}
\bm{x}_1\\
\vdots \\
\bm{x}_K
\end{array}\right)&=\mathbf{W}_{0}
\bm{x}/A,
\end{eqnarray}
where $A$ denotes the power normalization parameter, i.e., $A=\|\mathbf{W}_{0,((k-1)r:kr-1, :)}\bm{x}\|_2$. Hence the signal-noise ratio (SNR) of the $(k-1)r+t$-th element in $\bm{y}$ can be represented as
\begin{equation}
\textrm{SNR}_{f,(k-1)r+t}=\mathbb{E}_{\bm{x}}(\|\mathbf{C}^{H}_{t, :}\mathbf{H}\mathbf{P}\bm{x}_k\|_2)N_r/P_n,
\end{equation}
where $\mathbf{C}^{H}_{t, :}$ denotes the $t$-th row of $\mathbf{C}^{H}$ and $P_n$ is the total power of noise.

Symmetrically, the SNR of the backward transmission can be written by
\begin{equation}
\textrm{SNR}_{b,(k-1)r+t}=\mathbb{E}_{\bm{x}}(\|\mathbf{P}^{T}_{(t, :)}\mathbf{H}^{T}\mathbf{C}^*\bm{g}_{y,k}\|_2)N_t/(\tilde{A}P_n),
\end{equation}
where $\tilde{A}$ is the power normalization parameter of backward communication.

Then, we briefly show the relationship among different implementations of our system.
The transmitter parameterized designs are symmetric to the corresponding receiver parameterized designs, and the only difference is that the precoding matrices or the combining matrices are more complicated for combined designs. From the definition of SNR and the structure, $\textrm{SNR}_f$ and $\textrm{SNR}_b$ can be obtained similarly following the procedure above.

Naturally, the goal of MIMO communication is to maximize the SNR of the worst-used equivalent subchannel, i.e.,
\begin{equation}
\argmax_{\mathbf{P},\mathbf{C}} \min\{\min_{k,t}\textrm{SNR}_{f,(k-1)r+t}, \min_{k,t}\textrm{SNR}_{b,(k-1)r+t}\}.\label{target1}
\end{equation}
In the optimal solution to \eqref{target1}, each row of $\mathbf{P}$ and $\mathbf{C}$ is the singular vector corresponding to the largest singular value of channel matrix $\mathbf{H}$, i.e., both transmitter and receiver allocate full power on the optimal beamforming vector.
Considering that our system is a joint design of communication and split ML, we should give a thought to the constraints of split ML, i.e., the conditions of Lemma \ref{the::1} and Theorem \ref{the::2} hold.
However, the communication goal and the computation goal (the goal of NN) is coupled in the proposed system. We do not need to distinctly express the constraints of split ML in the overall loss of the NN since the computation goal of the system can implicitly ensure them.
Based on this fact, we show how to translate problem \eqref{target1} to the loss of NN in the following.
\subsection{Training Method of the Proposed System}
\begin{algorithm}\small
\caption{\textbf{Training Split ML via OAC}}
\label{alg::1}
\For {$n=1$ to $N$}{
\If {$n=1$}{
Sample a batch of input data $\bm{x}$\Comment{Begin forward computation}\\
$\bm{y}_0\leftarrow\bm{x}$\\
}
\Else{
Receive signal $\bm{y}$ by antennas from multiple transmissions\\
Calculate the covariance matrix: $\mathbf{R}_{f,n-1}\leftarrow\mathbb{E}(\bm{y}\bm{y}^H)$\\
Update the time-averaged covariance matrix: $\bar{\mathbf{R}}_{f,n-1}\leftarrow \alpha\bar{\mathbf{R}}_{f,n-1}+(1-\alpha)\mathbf{R}_{f,n-1}$\\
Obtain $\bm{y}_{n-1}$ with the combining matrices as shown in Fig. \ref{fig::system2}\\
}
Conduct local forward computation with $\bm{y}_{n-1}$, and outputs $\bm{x}_n$ for transmitting\\
\If {$n=N$}{Output $\bm{x}_n$ as the final prediction\\}
\Else{Transmit $\rm{s}_n$ with the precoding matrices as shown in Fig. \ref{fig::system2}\\
}}
Calculate the loss $\ell$ of the main task.\\
\For {$n=N$ to $1$}{
\If{$n\neq N$}{
Receive signal $\bm{g}_s$ by antennas from multiple transmissions\Comment{Begin backward propagation}\\
Calculate the covariance matrix: $\mathbf{R}_{b,n}\leftarrow\mathbb{E}(\bm{g}_x\bm{g}_x^H)$\\
Update the time averaged covariance matrix: $\bar{\mathbf{R}}_{b,n}\leftarrow \alpha\bar{\mathbf{R}}_{b,n}+(1-\alpha)\mathbf{R}_{b,n}$\\
Cast SVD decomposition on $\bar{\mathbf{R}}_{b,n}$: $\bar{\mathbf{R}}_{b,n} = \mathbf{U}\bm{\Sigma}\mathbf{V}^H$\\
Add the goal of communication: $\bm{g}_x\leftarrow\bm{g}_x+(1/2B^2)\partial\|\bm{x}_n^H\mathbf{V}(:, N_t-r+1:N_t)\|_2^2/\partial{\bm{x}}$\\
Obtain $\bm{g}_{x,n}$ with the precoding matrices as shown in Fig. \ref{fig::system2}\\
Employ backward propagation from $\bm{g}_{x,n}$, and obtain $\bm{g}_{y,n-1}$\\
}
\Else{Employ backward propagation from $\ell$, and obtain $\bm{g}_{y,n-1}$\\}
\If{$n\neq 1$}{
SVD decomposition on $\bar{\mathbf{R}}_{f,n-1}$: $\bar{\mathbf{R}}_{f,n-1} = \mathbf{U}\bm{\Sigma}\mathbf{V}^H$\\
Transmit $\rm{g}_y$ with the combining matrices as shown in Fig. \ref{fig::system2}\\
Add the goal of communication: $\bm{g}_C\leftarrow\bm{g}_C+\partial\|\mathbf{C}\mathbf{U}(:, N_r-r+1:N_r)\|_2^2/\partial{\mathbf{C}}$
}
}
Each node updates local parameters based on the calculated gradients.
\end{algorithm}
In this section, we introduce a method to translate \eqref{target1} to the loss of NN.
As explained above, simply considering the goal in \eqref{target1} leads to an obviously unreasonable solution. Moreover, the goal in \eqref{target1} is not derivable, making it impossible to be directly applied in NNs. However, it inspires us that in the optimal system, we should only use some of the best subchannels.
Based on this inspiration, we can find conditions for suboptimal solutions of \eqref{target1} that
\begin{eqnarray}
\mathbf{H} = \mathbf{U}\bm{\Sigma}\mathbf{V}^H,\label{SVD}\\
\mathbf{P}\mathbf{V}_{:, -r:} = \mathbf{0},\label{loss1}\\
\mathbf{C}\mathbf{U}_{:, -r:} = \mathbf{0},\label{loss2}
\end{eqnarray}
where \eqref{SVD} is the SVD decomposition of $\mathbf{H}$, and $\mathbf{V}_{:, -r:}$ denotes the matrix consisting of the last $r$ columns of $\mathbf{V}$, i.e., the singular vectors corresponding to the smallest $r$ singular values.

The conditions mainly come from Lemma \ref{the::1} that $r$ is the minimum value that ensures the conditions of Lemma \ref{the::1} and the conditions above could be satisfied simultaneously. Under the conditions above, the columns of the matrices $\mathbf{P}_k$ and $\mathbf{C}_k$ are constrained to linear spaces of dimension $r$, respectively. According to the basic linear algebra theory, there exist $\mathbf{P}_k$s and $\mathbf{C}_k$s satisfying both constraints. We note that in the final trained NNs, the ranks $\mathbf{P}_k$ and $\mathbf{C}_k$ may not always be $r$. However, it is caused by the sparsity of NN parameters \cite{sparse}, which does not mean a mistake in our design.

In real systems, we cannot directly obtain $\mathbf{H}$ directly. However, we can still realize the constraints in a similar way with self-adaptive beamforming. In simple self-adaptive beamforming algorithms, we usually use the covariance matrix of the received signal to estimate the singular vectors of the channel $\mathbf{H}$ and hence obtain the optimal beamforming vector. Similarly, we can use the singular vectors of the covariance matrix of the received signal to estimate the singular vectors and then compute the loss by \eqref{loss1} and \eqref{loss2}. Combining \eqref{loss1}, \eqref{loss2}, and the training algorithm of centralized NNs, we finally obtain the detailed training algorithm for each batch of input data as shown in Alg. \ref{alg::1}.

\section{Generalization and Comparison}\label{sec::4}
In this section, we generalize the proposed MIMO OAC-based split ML system to convolutional NNs. We also analyze and compare the different implementations of the proposed system.
\subsection{Extension to Convolutional NNs}\label{sec::CNN}
For the considered convolutional layer, we denote the size of the input and the output images as $N_{wi}\times N_{hi}$ and $N_{wo}\times N_{ho}$ with $N_{ci}$ and $N_{co}$ channels, respectively. The size of the convolutional kernel is $N_k\times N_k$.

\begin{figure*} [!htp]
    \centering
       \includegraphics[width=0.9\linewidth]{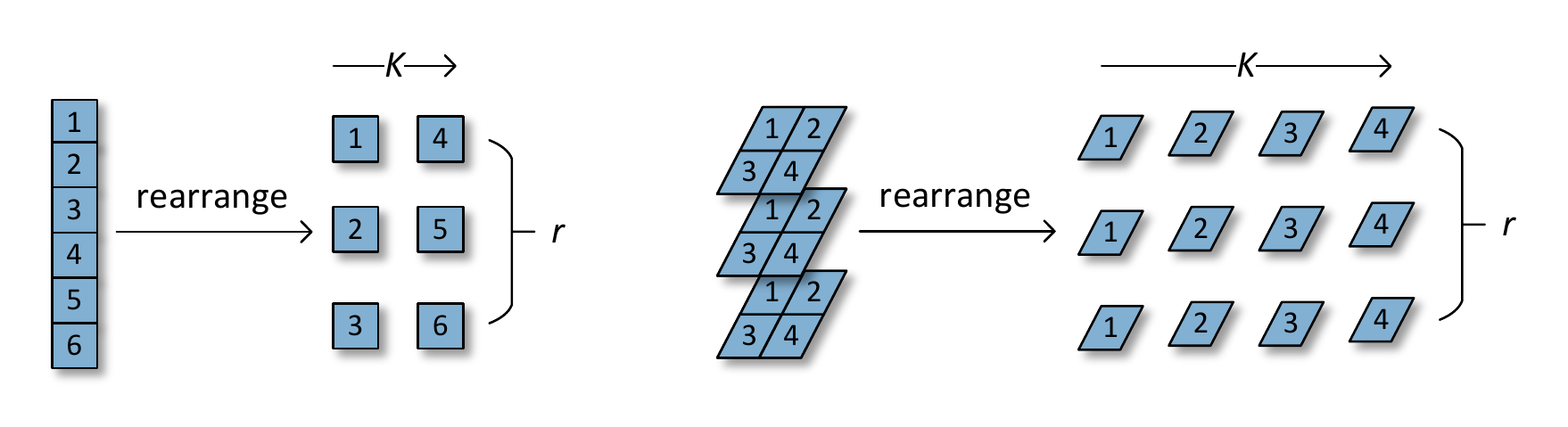}
	  \caption{Illustration of the rearrangement in Fig. \ref{fig::system2} for both full-connected (left) and convolutional (right) layers.}
	  \label{fig::system3}
\end{figure*}
Here, we show how to extend the implementations to convolutional layers. The calculation in a convolutional layer can be described as:
\begin{equation}
\left(\begin{array}{c}
\mathbf{Y}_{1} \\
\vdots \\
\mathbf{Y}_{N_{ci}}
\end{array}\right)
=\left(\begin{array}{ccc}
\mathbf{K}_{1} & \cdots & \mathbf{K}_{1} \\
\vdots & \ddots & \vdots \\
\mathbf{K}_{N_{co}} & \cdots & \mathbf{K}_{N_{co}}
\end{array}\right) *\left(\begin{array}{c}
\mathbf{X}_{1} \\
\vdots \\
\mathbf{X}_{N_{ci}}
\end{array}\right),
\end{equation}
where $\mathbf{X}_{n}$, $\mathbf{Y}_{n}$, $\mathbf{K}_{n}$ represent the $n$-th input channel, the $n$-th output channel, and the $n$-th convolutional kernel, respectively. The symbol $*$ represents convolution operation, and the calculation rules between matrices are the same as matrix multiplication, except that convolution replaces multiplication.

Since the convolution operation can be seen as a particular multiplication with specific rules, it is still a linear operation. Hence it still satisfies the associative laws the. Introducing an arbitrary matrix $\mathbf{W}\in\mathcal{C}^{N_{ci}*N_{co}}$, we have
\begin{equation}
\small
\begin{aligned}
&\left((\mathbf{W}\otimes\mathbf{I}_{N_k \times N_k})\left(\begin{array}{ccc}
\mathbf{K}_{1} & \cdots & \mathbf{K}_{1} \\
\vdots & \ddots & \vdots \\
\mathbf{K}_{N_{co}} & \cdots & \mathbf{K}_{N_{co}}
\end{array}\right)\right) *\left(\begin{array}{c}
\mathbf{X}_{1} \\
\vdots \\
\mathbf{X}_{N_{ci}}
\end{array}\right)\\
=&
(\mathbf{W}\otimes\mathbf{I}_{N_{ho} \times N_{ho}})\left(\left(\begin{array}{ccc}
\mathbf{K}_{1} & \cdots & \mathbf{K}_{1} \\
\vdots & \ddots & \vdots \\
\mathbf{K}_{N_{co}} & \cdots & \mathbf{K}_{N_{co}}
\end{array}\right) *\left(\begin{array}{c}
\mathbf{X}_{1} \\
\vdots \\
\mathbf{X}_{N_{ci}}
\end{array}\right)\right).
\end{aligned}
\label{eqcnn}
\end{equation}

Equation \eqref{eqcnn} shows that multiplying a matrix $(\mathbf{W}\otimes\mathbf{I}_{N_k \times N_k})$ after convolution is equivalent to multiplying the matrix on the kernels or applying another convolutional layer with $1\times 1$ kernels, i.e.,
\begin{equation}
\footnotesize
\begin{aligned}
&\left(\begin{array}{ccc}
\mathbf{K}'_{1} & \cdots & \mathbf{K}'_{1} \\
\vdots & \ddots & \vdots \\
\mathbf{K}'_{N_{co}} & \cdots & \mathbf{K}'_{N_{co}}
\end{array}\right)\\
&=
(\mathbf{W}\otimes\mathbf{I}_{N_k \times N_k})\left(\begin{array}{ccc}
\mathbf{K}_{1} & \cdots & \mathbf{K}_{1} \\
\vdots & \ddots & \vdots \\
\mathbf{K}_{N_{co}} & \cdots & \mathbf{K}_{N_{co}}
\end{array}\right).
\end{aligned}
\end{equation}

Hence, for convolutional layers, we only need to replace the rearranging method in Fig. \ref{fig::system2}. The difference between the rearranging methods of full-connected and convolutional layers is shown in Fig. \ref{fig::system3}. We note that, in Fig. \ref{fig::system3}, only the intermediate output corresponding to a piece of input data is shown, whereas a batch of outputs is transmitted in each iteration in practice.
\subsection{Comparisons Among the Implementations}\label{sec::comp}
In this section, we compare all the proposed implementations in terms of the amounts of parameters, calculation, and communication,
which are described by the number of complex parameters, the number of complex multiply-accumulate operations (MACs), and the number of forward transmissions per batch, respectively.
The MAC operation is defined as calculating $a\leftarrow a+b\times c$ by a multiplier–accumulator unit, which is widely adopted in the hardware field. We show the results of the comparison in Table \ref{table::1}.
\begin{table*}[]
\centering
\caption{The amounts of parameters, computation, and communication of all designs per batch of data. (``FC'' is for fully connected layers, and ``Conv'' is for convolutional layers. $B$ represents the batch size.)}\label{table::1}
\begin{tabular}{|c|c|c|c|c|}
\cline{1-5}
 & Implementation & Parameters & MACs & Transmissions \\ \cline{1-5}
\multirow{4}{*}{FC} & \begin{tabular}[c]{@{}c@{}}Transmitter parameterized,\\ combined design\end{tabular} & $N_iN_oN_t/r+N_rr$ & $BN_o(N_r+N_iN_t/r)$ & $BN_o/r$ \\ \cline{2-5}
 & \begin{tabular}[c]{@{}c@{}}Transmitter parameterized,\\ separated design\end{tabular} & $N_iN_o+(N_t+N_r)r$ & $BN_o(N_i+N_t+N_r)$ & $BN_o/r$ \\ \cline{2-5}
 & \begin{tabular}[c]{@{}c@{}}Receiver parameterized,\\ combined design\end{tabular} & $N_iN_oN_r/r+N_tr$ & $BN_i(N_t+N_oN_r/r)$ & $BN_i/r$ \\ \cline{2-5}
 & \begin{tabular}[c]{@{}c@{}}Receiver parameterized,\\ separated design\end{tabular} & $N_iN_o+(N_t+N_r)r$ & $BN_i(N_o+N_t+N_r)$ & $BN_i/r$ \\ \cline{1-5}
\multirow{4}{*}{Conv} & \begin{tabular}[c]{@{}c@{}}Transmitter parameterized,\\ combined design\end{tabular} & $N_{co}N_k^2N_t/r+N_rr$ & \begin{tabular}[c]{@{}c@{}}$BN_{ci}N_{co}N_{wo}N_{ho}N_k^2N_t/r$\\ $+BN_{co}N_{wo}N_{ho}N_r$\end{tabular} & $BN_{co}N_{wo}N_{ho}/r$ \\ \cline{2-5}
 & \begin{tabular}[c]{@{}c@{}}Transmitter parameterized,\\ separated design\end{tabular} & $N_{co}N_k^2+(N_t+N_r)r$ & \begin{tabular}[c]{@{}c@{}}$BN_{ci}N_{co}N_{wo}N_{ho}N_k^2$\\ $+BN_{co}N_{wo}N_{ho}(N_t+N_r)$\end{tabular} & $BN_{co}N_{wo}N_{ho}/r$ \\ \cline{2-5}
 & \begin{tabular}[c]{@{}c@{}}Receiver parameterized,\\ combined design\end{tabular} & $N_{co}N_k^2N_r/r+N_tr$ & \begin{tabular}[c]{@{}c@{}}$BN_{ci}N_{co}N_{wo}N_{ho}N_k^2N_r/r$\\ $+BN_{ci}N_{wi}N_{hi}N_t$\end{tabular} & $BN_{ci}N_{wi}N_{hi}/r$ \\ \cline{2-5}
 & \begin{tabular}[c]{@{}c@{}}Receiver parameterized,\\ separated design\end{tabular} & $N_{co}N_k^2+(N_t+N_r)r$ & \begin{tabular}[c]{@{}c@{}}$BN_{ci}N_{co}N_{wo}N_{ho}N_k^2$\\ $+BN_{ci}N_{wi}N_{hi}(N_t+N_r)$\end{tabular} & $BN_{ci}N_{wi}N_{hi}/r$ \\ \cline{1-5}
\end{tabular}
\end{table*}

We note that $r$ is always no more than $\min\{N_r, N_t\}$, even probably being far smaller than $\min\{N_r, N_t\}$ under sparse channels.
From this fact and Table \ref{table::1}, we know that for most cases, although the separated design requires one more time of matrix multiplication, it usually uses fewer parameters and costs less in computation because of smaller matrices. Moreover,  whether the transmitter parameterized designs or the receiver parameterized designs are better depends on the relationship between the input and output sizes. If the input size is larger than the output size, the receiver parameterized designs are better, and vice versa.
\section{Numerical Results}\label{sec::5}
In this section, we evaluate the performance of the proposed system and provide numerical results based on the CIFAR-10 image classification dataset.
We use the full CIFAR 10 dataset with 50000 training images and 10000 testing images in 10 categories. All the results are reported in terms of testing accuracy.
\subsection{Experimental Settings}
\subsubsection{Communication System Settings}
We adopt common MIMO communication settings, where the channel is composed of $N_p$ different paths, as for the gain, transmission direction angle, and arrival direction angle of each path are independently and randomly generated by uniform distribution, i.e.,
\begin{equation}\label{channel}
\mathbf{H} = \sum_{n=1}^{N_p} a_n (1,\cdots,e^{j(N_r-1)\theta_n})^H(1,\cdots,e^{j(N_t-1)\phi_n}),
\end{equation}
where $\theta_n\sim U(-\pi,\pi)$, $\phi_n\sim U(-\pi,\pi)$, $|a_n|\sim U(0.5, 1.5)$, and $\textrm{angle}(a_n)\sim U(-\pi,\pi)$ for each $n$.

The total transmitting power of each transmission is normalized to 1, the noise power of each receiving antenna is the same, and the SNR is defined by
\begin{equation}
\textrm{SNR}=\|\mathbf{H}\|_2 N_r/P_n,
\end{equation}
where $P_n$ is the total noise power.

We consider the following settings in this section.
Firstly, we consider a 2-node complex channel setting, where $N_t=N_r=16$ and $N_p=20$, such that the channel is with full rank.
We also consider a 3-node sparse channel setting, where $N_t=N_r=16$ and $N_p=4$ for both channels.
Considering the massive MIMO settings, in reality, we consider a massive MIMO setting where there are 3 nodes, and for both channels, $N_t=N_r=64$ and $N_p=8$.
Finally, we consider a 3-node moving setting where $N_t=N_r=16$ and $N_p=4$ for both channels. To ensure fairness, the channels of all settings are pre-generated and remain the same in all experiments.
\subsubsection{NN Settings}
\begin{figure*}
    \centering
       \includegraphics[width=0.75\linewidth]{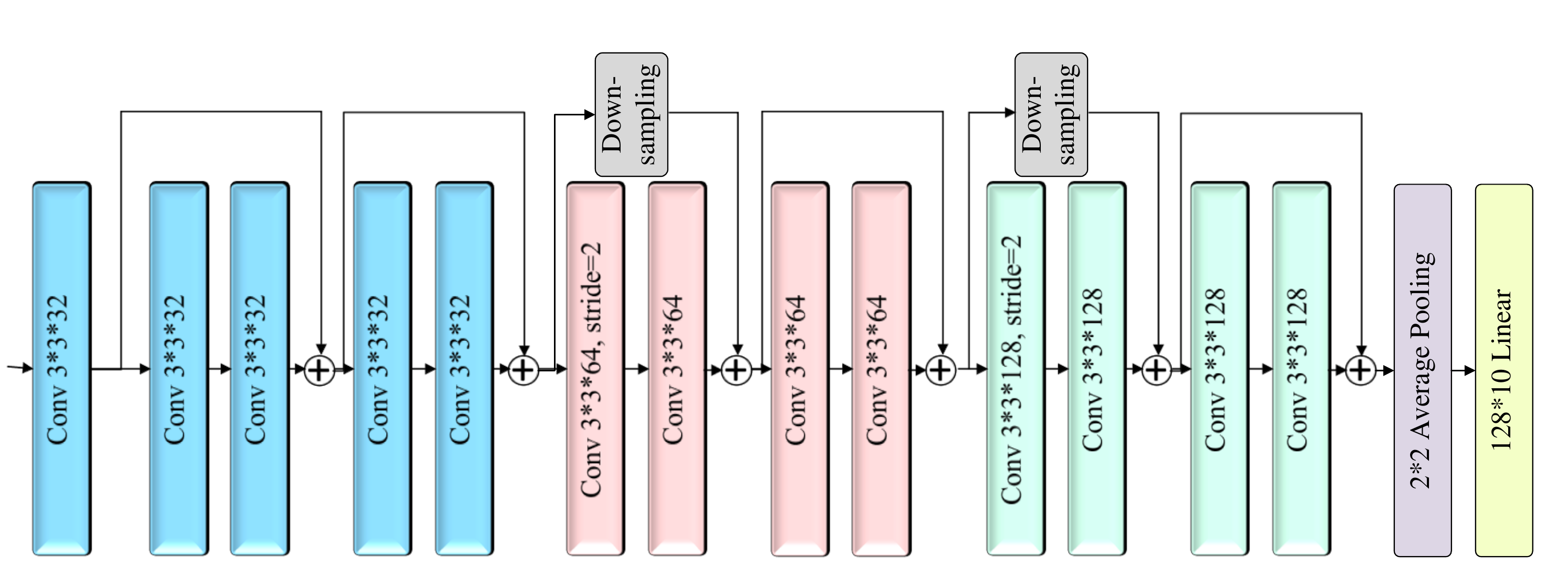}
       \includegraphics[width=0.75\linewidth]{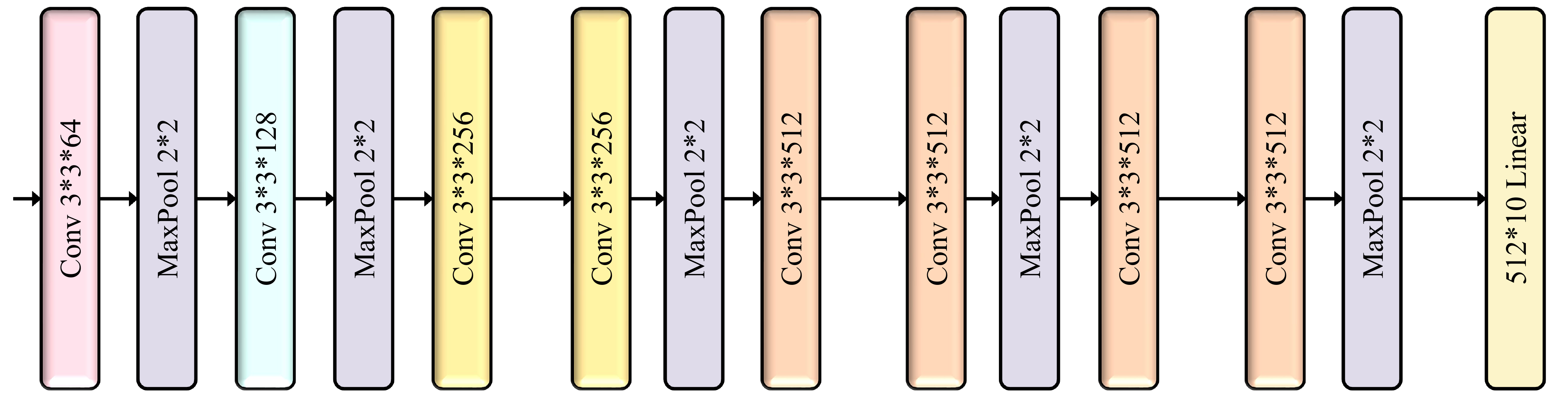}
	  \caption{The structures of the deployed complex ResNet (top) and VGGNet (bottom).}
	  \label{fig::NN}
\end{figure*}
We use the 16-layer complex ResNet to evaluate the performance of our proposed system and algorithms in this section. The NN structure follows the one in \cite{complexResNet} and is shown in Fig. \ref{fig::NN}. ``Conv$3*3*N_c$, stride=2'' refers to a convolutional layer with $N_c$ kernels with size $3*3$ and stride 2, and stride is 1 if not written. ``Downsampling'' refers to cast downsampling such that the outputs of the shortcut connection and the convolutional layers have the same size, which is realized by the convolutional layer with $1*1$ kernels. ``Average pooling'' and ``Linear'' stand for the corresponding layer with the following parameters. Each layer, except the pooling layer, is followed by a BN layer.
For 2-node settings, the NN is split at the 9th layer, and for 3-node settings, the NN is split at the 5th and 13th layers.
We also consider a VGGNet structure as also shown in Fig. \ref{fig::NN}. It is another famous convolutional neural network structure that is wider but shallower than the ResNet.
For 2-node settings, the VGGNet is split at the 5th layer, and for 3-node settings, the VGGNet is split at the 3rd and 8th layers.
We use batch size 64 and Adam optimizer \cite{adam} with the learning rate of 0.005 to train the NN, and set $\alpha=0.99$.
Since the output sizes of the convolutional layers are no more than the corresponding input sizes, we use the receiver parameterized designs. According to Table \ref{table::1}, we choose the separated design for $r<16$ and the combined design for $r=16$ for the minimum computation and storage cost.

\subsection{Comparison Algorithms}
We verify the proposed system and algorithm by comparing them with the straightforward algorithms below. The first is the centralized algorithm, where the NN is trained in a centralized manner. We also use traditional split ML and MIMO-based split ML, where the communication of split ML is conducted by traditional MIMO. We note that both of them are identical to the centralized algorithm in performance since, in such systems, communication is just a tube.
The last comparison algorithm is called the ideal case, where the channel is assumed to be known exactly. Hence, from \eqref{SVD}-\eqref{loss2}, we can use the following precoding and combining matrix:
\begin{eqnarray}
\mathbf{P}^{H}=\mathbf{V}_{:, :N_r-r},\\
\mathbf{C}^{H}=\mathbf{U}_{:, :N_r-r},
\end{eqnarray}
where the matrices are calculated by \eqref{SVD}. When training NNs, the matrices $\mathbf{P}$ and $\mathbf{C}$ is determined by the channel and is not trainable in the whole progress.

The costs of all considered algorithms are compared in Table \ref{table::2}, where the smallest value in each column is marked in bold font. We assume that 16 QAM modulation is used in digital systems. Hence a 64-bit complex float number is represented by 16 digital symbols or an analog symbol. Moreover, digital communication systems require channel coding and retransmissions according to channel quality.
Thus, analog communication requires no more than 1/16 times of transmission compared with digital communication.
In Table \ref{table::2}, $C_\textrm{NN}$, $C_\textrm{Trans}$, and $C_\textrm{CE}$ are defined as the total cost of calculating the NN (in terms of MACs), transmission (in terms of communication rounds), and channel estimation (in terms of times) in traditional split ML, respectively. $C_\textrm{MIMO}$ represents the total native computation cost by the MIMO system (precoding and combining, in terms of MACs) in traditional ML. According to the calculation in Table \ref{table::1}, we always have $C_\textrm{MIMO}<<C_\textrm{NN}$.
We emphasize that our MIMO-based OAC design greatly reduces the transmission cost compared to traditional split ML without OAC.
It also saves the system costs and/or improves its overall efficiency by eliminating explicit channel estimation when compared with the ideal case.
According to Table \ref{table::2}, the proposed algorithm yields the lowest cost, which reflects the high communication and transmission efficiency of the considered MIMO-based OAC structure.
\begin{table}[]\centering
\caption{Cost comparisons among different algorithms.}
\label{table::2}
\begin{tabular}{|p{0.1\textwidth}<{\centering}|c|p{0.1\textwidth}<{\centering}|p{0.1\textwidth}<{\centering}|}
\hline
 & Computation & Transmission & Channel Estimation \\ \hline
Traditional Split ML & \multirow{2}{*}{$\bm{C}_\textrm{NN}$} & \multirow{2}{*}{$C_\textrm{Trans}$} & \multirow{2}{*}{$C_\textrm{CE}$} \\ \hline
MIMO-Based Split ML & \multirow{2}{*}{$C_\textrm{NN}+C_\textrm{MIMO}$} & \multirow{2}{*}{$C_\textrm{Trans}/r$} & \multirow{2}{*}{$C_\textrm{CE}/r$} \\ \hline
\multirow{2}{*}{Ideal Case} & \multirow{2}{*}{$C_\textrm{NN}+C_\textrm{MIMO}$} & \bf{at most} $\bm{C}_\textrm{Trans}/(\bm{16r})$ & at most $C_\textrm{CE}/(16r)$ \\ \hline
\multirow{2}{*}{Proposed} & \multirow{2}{*}{$C_\textrm{NN}+C_\textrm{MIMO}$} & \bf{at most} $\bm{C}_\textrm{Trans}/(\bm{16r})$ & \multirow{2}{*}{\bf{0}} \\ \hline
\end{tabular}
\end{table}
\subsection{Results Under Stable Complex Channel}\label{sec::complex}
In this section, we consider the 2-node complex channel setting. In the setting, the channel is almost of full rank since there are 20 independent randomly generated paths. We suppose the pre-estimation of the rank $r$ to be 4, 8, 12, or 16 to verify the performance of the algorithms.

Fig. \ref{fig::complex_t} shows the convergence trend of various algorithms of the ResNet. In this figure, we consider a high-quality channel, where SNR = 35 dB and $r=16$. In Fig. \ref{fig::complex_t}, both the ideal case and the proposed algorithm do not show a significant difference in terms of convergence rate, but with a slight performance loss, which verifies our analysis that the proposed system implementation is equivalent to the centralized ML system if there is almost no noise. From the centralized case, we can also find both structures have similar learning abilities.

In Figs. \ref{fig::complex} to \ref{fig::complex_SNR_alg}, we present the effect of SNR with different pre-estimation $r$ and the ResNet.
We can find that the trends of all the curves are similar that the accuracy increases as SNR increases and obtain almost centralized performance when the channel is of high quality, which is intuitive.

For the same method with different $r$s, although the pre-estimation $r=16$ is the most accurate, smaller $r$ lead to better performance. This is because when using smaller $r$, smaller data size is transmitted per transmission, and fewer subchannels are used, both equivalent to higher SNR.

We note that when $r=16$, all possible subchannels are used, and hence the optimization for communication does not valid. For other values of $r$, we can find that the proposed algorithm performs better with low SNR, whereas the ideal case performs better with high SNR.
This is mainly because when the channel is well, the best strategy is to simply transmit all the outputs to the receiver. Hence the ideal case performs better.
However, when the channel becomes poor, the best strategy becomes to stress transmitting only a part of the outputs to obtain higher quality. This strategy can make sense mainly because the output of NNs is usually sparse \cite{sparse}, and the performance loss caused by the reduction of width is smaller than that by the increase of the noise on the intermediate output.

In Fig. \ref{fig::networks}, we show the results of the proposed algorithm for both networks with $r=16$.
The results show that the performance of VGGNet is better than that of ResNet, especially that its performance drops later as the channel becomes worse.
That is probably because the VGGNet is wider, requiring more communication load in each transmission, which is usually detrimental for split ML.
However, in the proposed scheme, wider NNs can retain more information in the intermediate results, which is somehow similar to retransmission.
From these results, the neural network structure suitable for the proposed split SL system needs further discussion.

\begin{figure}[!htb]
\centering
\includegraphics[width=0.455\textwidth]{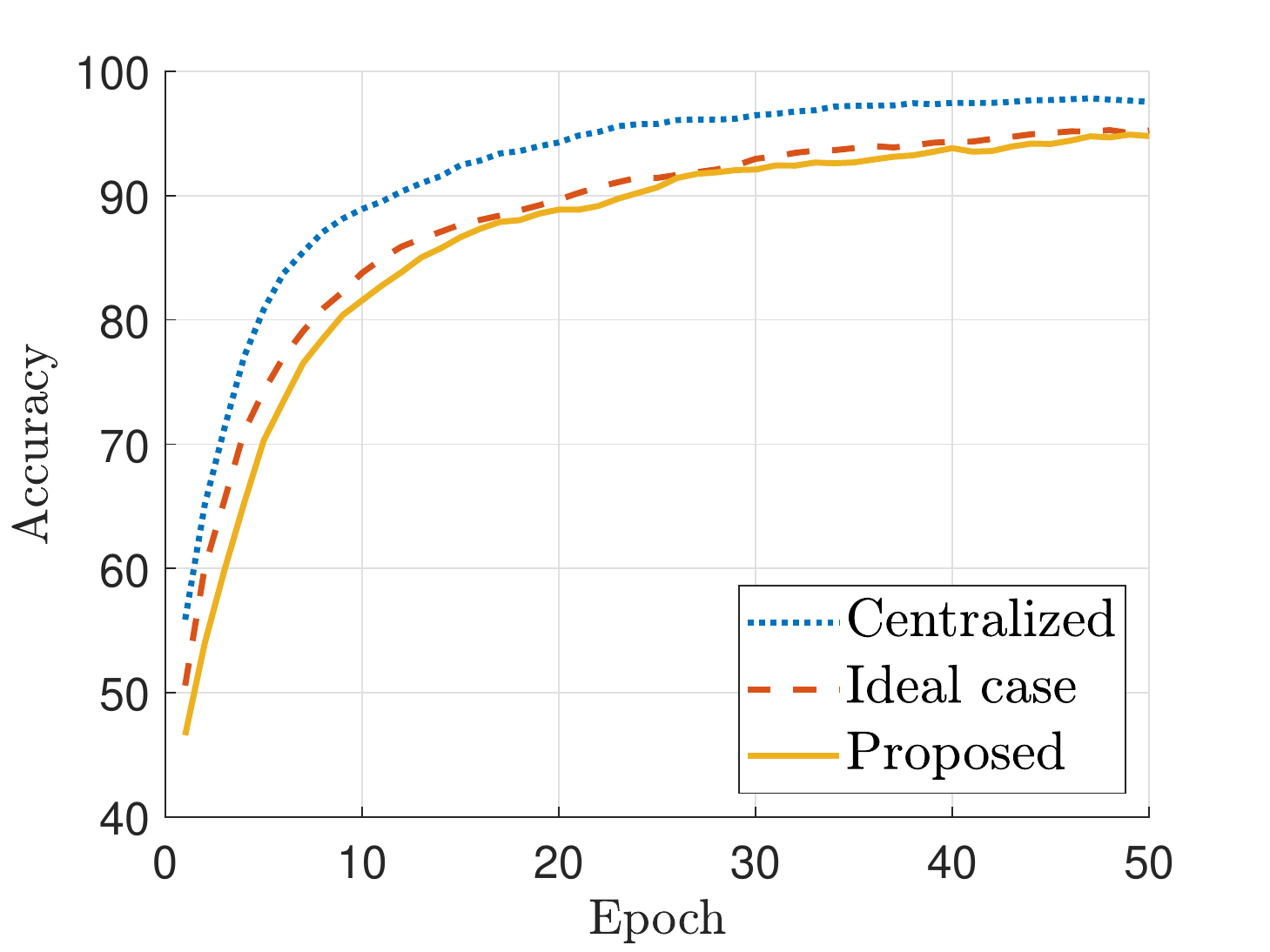}
\caption{The convergence curve of different algorithms.}\label{fig::complex_t}
\centering
\includegraphics[width=0.455\textwidth]{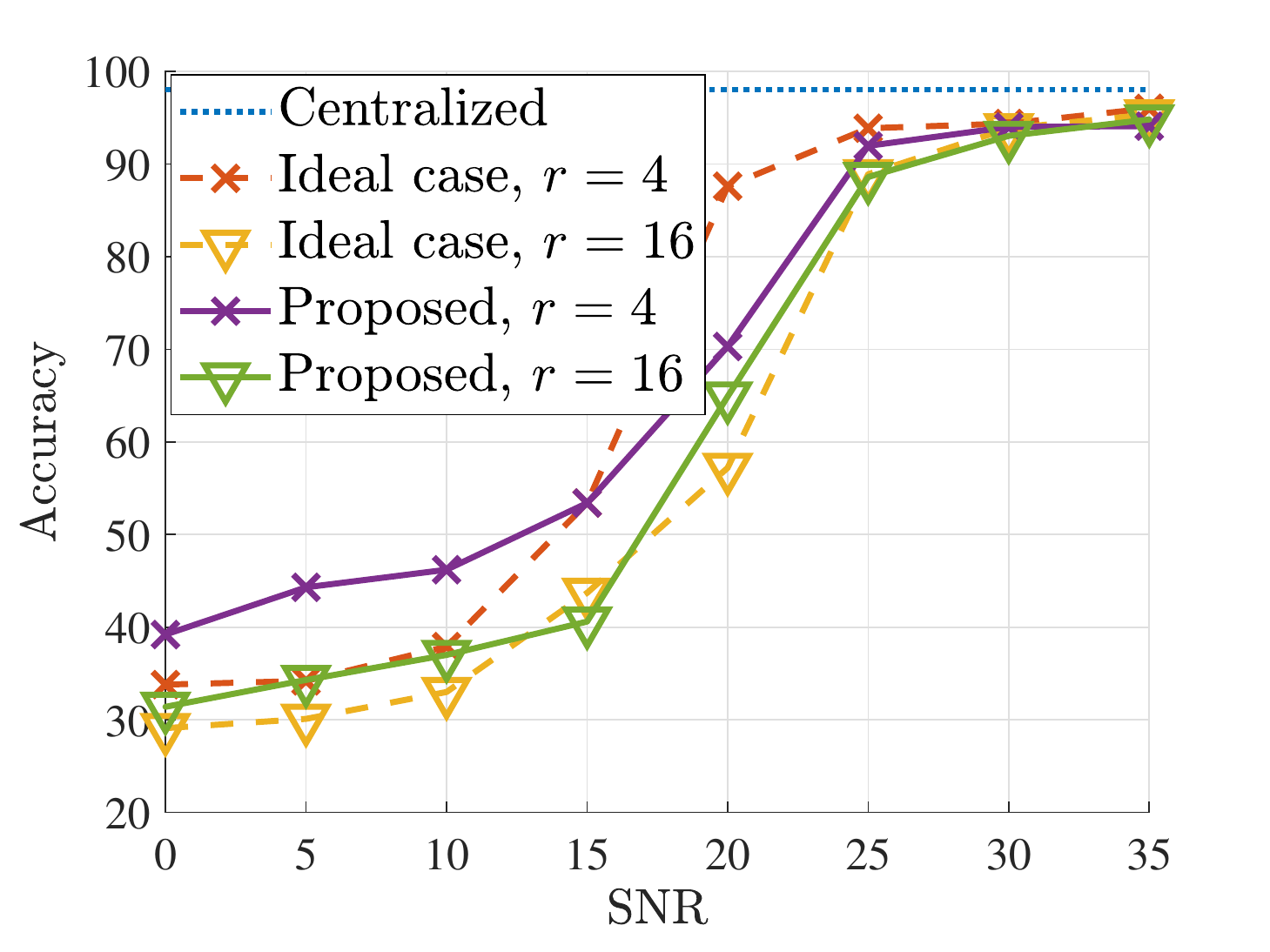}
\caption{The performance of different algorithms with different $r$s under the stable complex channel.}\label{fig::complex}
\end{figure}
\begin{figure}[!htb]
\centering
\includegraphics[width=0.455\textwidth]{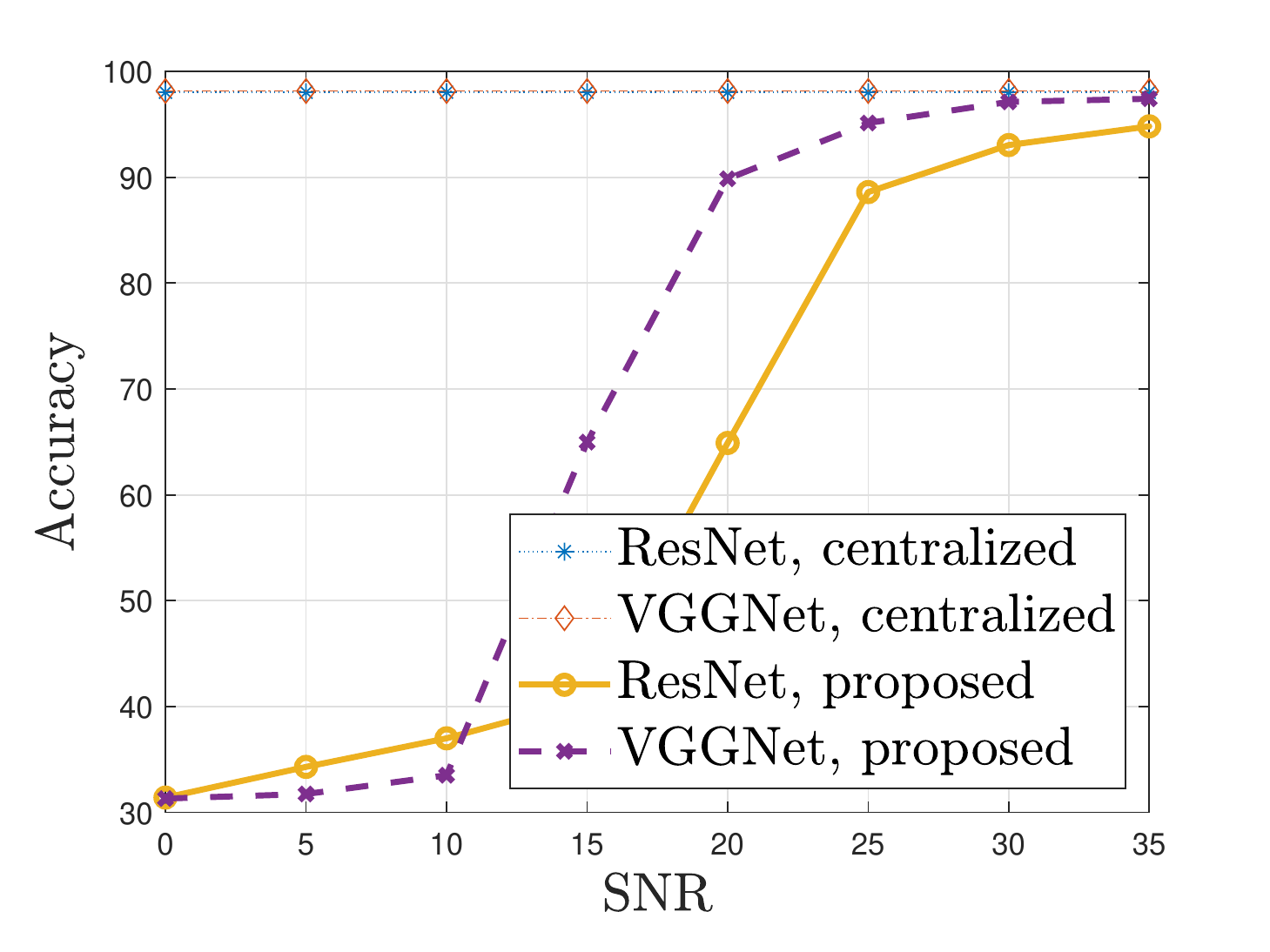}
	  \caption{The performance of different NNs with different SNRs.}\label{fig::networks}
\end{figure}

\begin{figure}[!htb]
\centering
\includegraphics[width=0.455\textwidth]{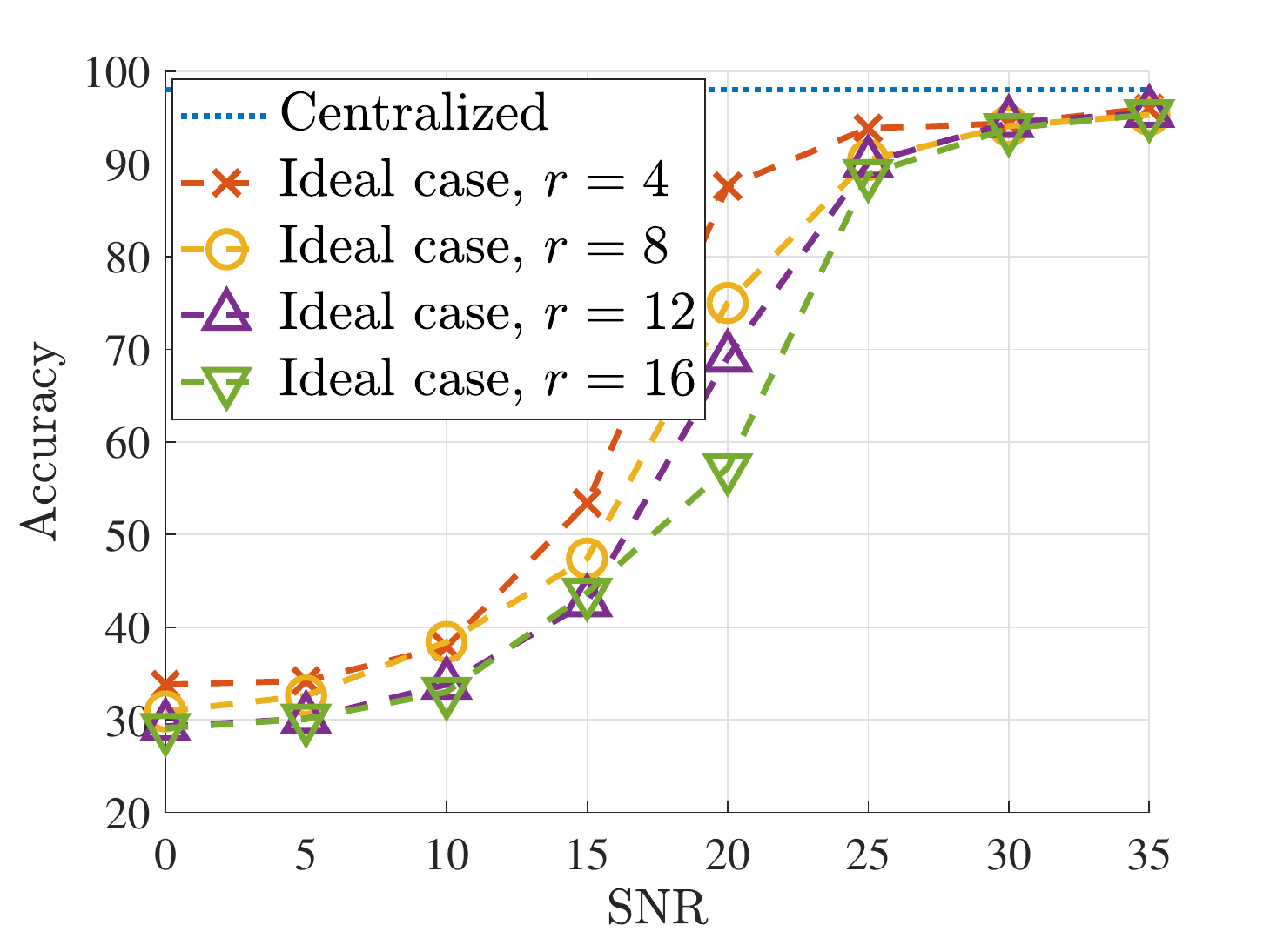}
\caption{The performance of the ideal case with different $r$s under the stable complex channel.}\label{fig::complex_SNR_baseline}
\centering
\includegraphics[width=0.455\textwidth]{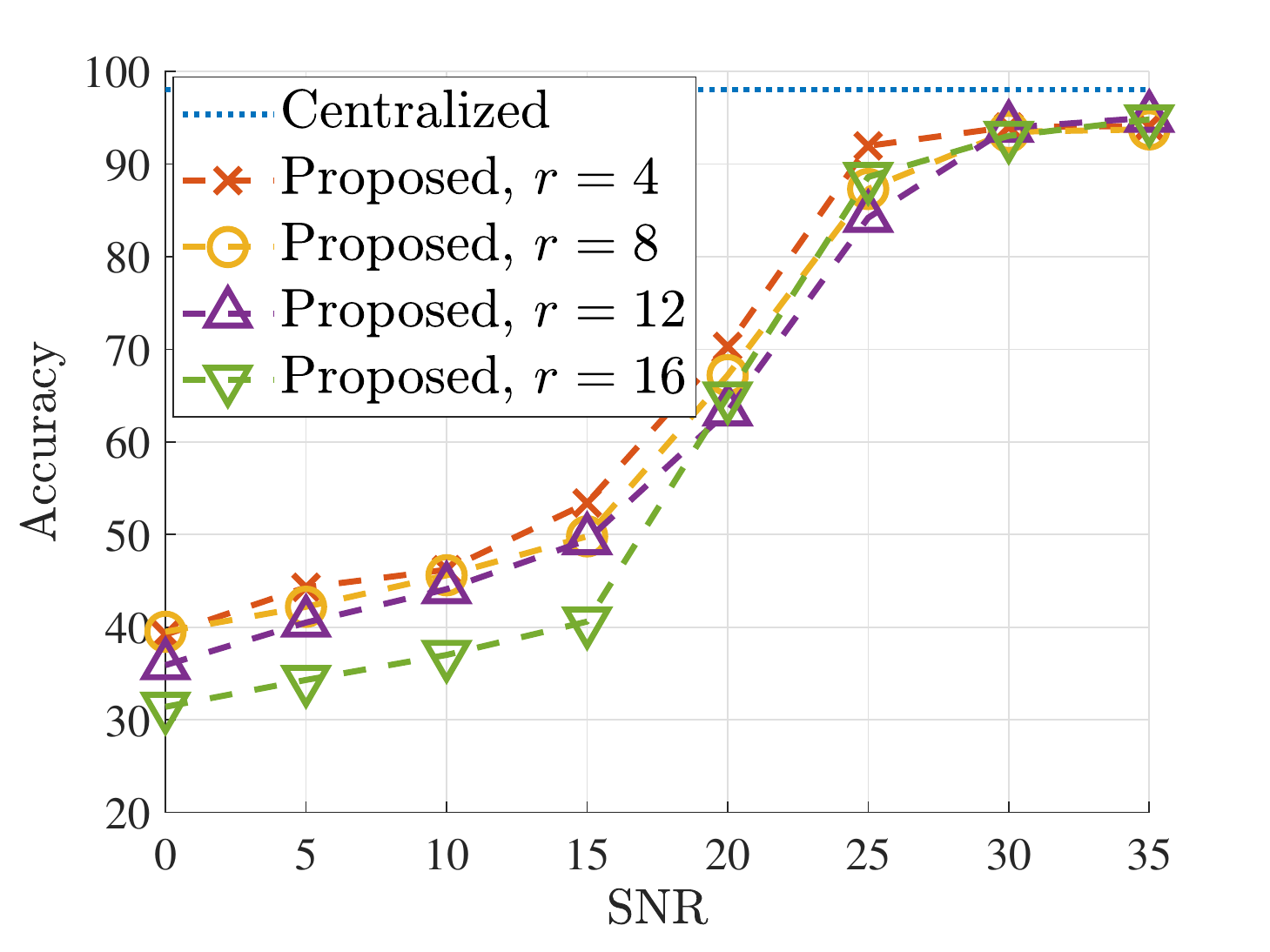}
\caption{The performance of the proposed algorithm with different $r$s under the stable complex channel.}\label{fig::complex_SNR_alg}
\end{figure}

\subsection{Results Under Stable Sparse Channels}\label{sec::sparse}
In this section, we consider the 3-node sparse channel setting with ResNet. In the setting, the channel is of rank 4. We suppose the pre-estimation of the rank $r$ to be 4 or 8 to verify the performance of the algorithms.
In Fig. \ref{fig::sparse}, we show the performance of different algorithms and different pre-estimations of $r$.
The results and conclusions are similar to those of complex channels. However, it is counterintuitive that the performance when $r=8$ does not decline sharply, although it does not satisfy the conditions of Theorem \ref{the::2}. This phenomenon is mainly caused by the sparsity of NNs \cite{sparse}. All the preconditions are derived by the suppose that the parameters $\mathbf{W}$ can take arbitrary values in $\mathcal{C}^{N_i \times N_o}$, whereas the optimal $\mathbf{W}$ is usually sparse in reality, which may explain the small performance loss.

\begin{figure}[!htb]
\centering
\includegraphics[width=0.455\textwidth]{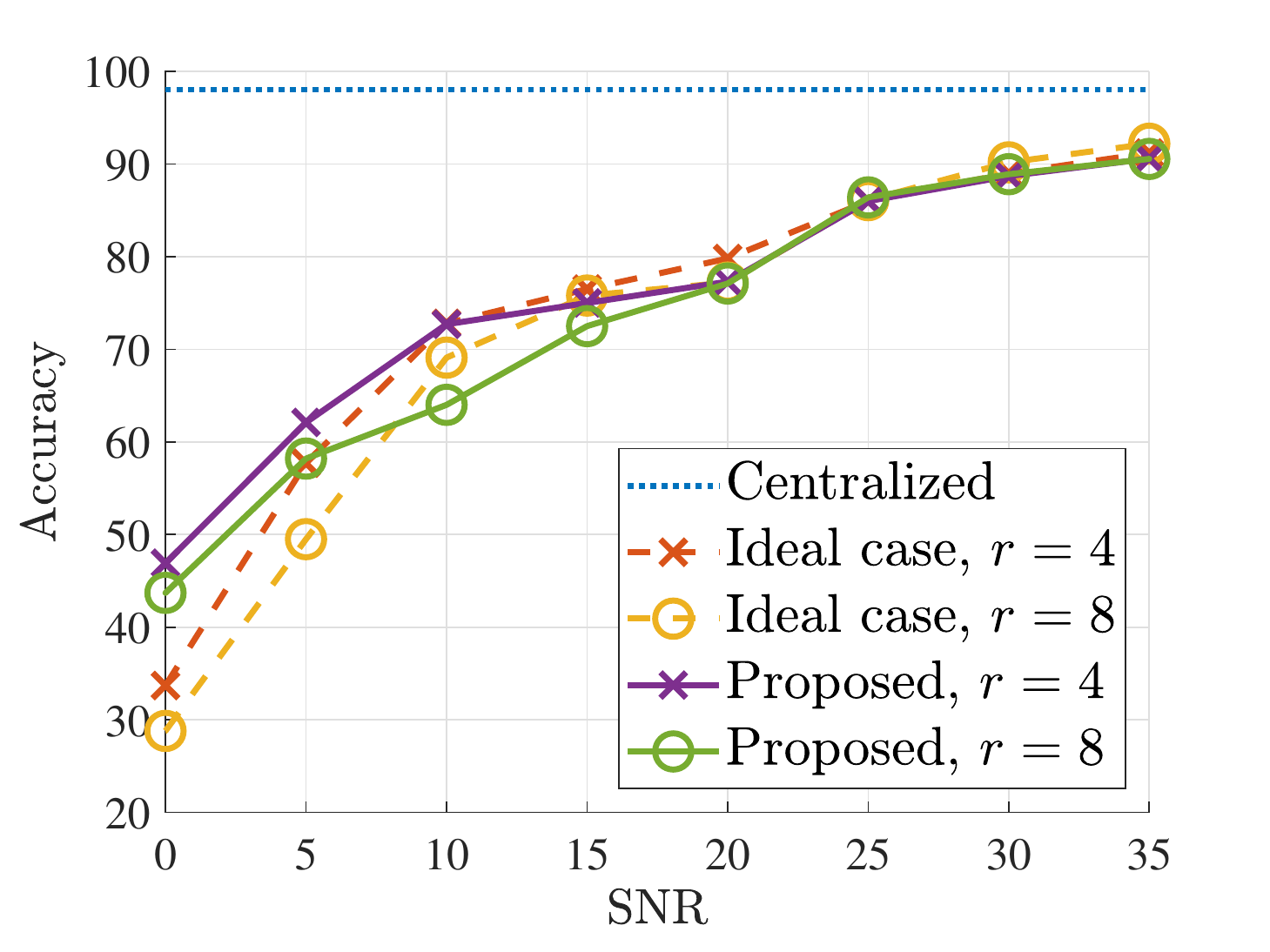}
	  \caption{The performance of different algorithms with different $r$s under stable sparse channels.}\label{fig::sparse}
\centering
\includegraphics[width=0.455\textwidth]{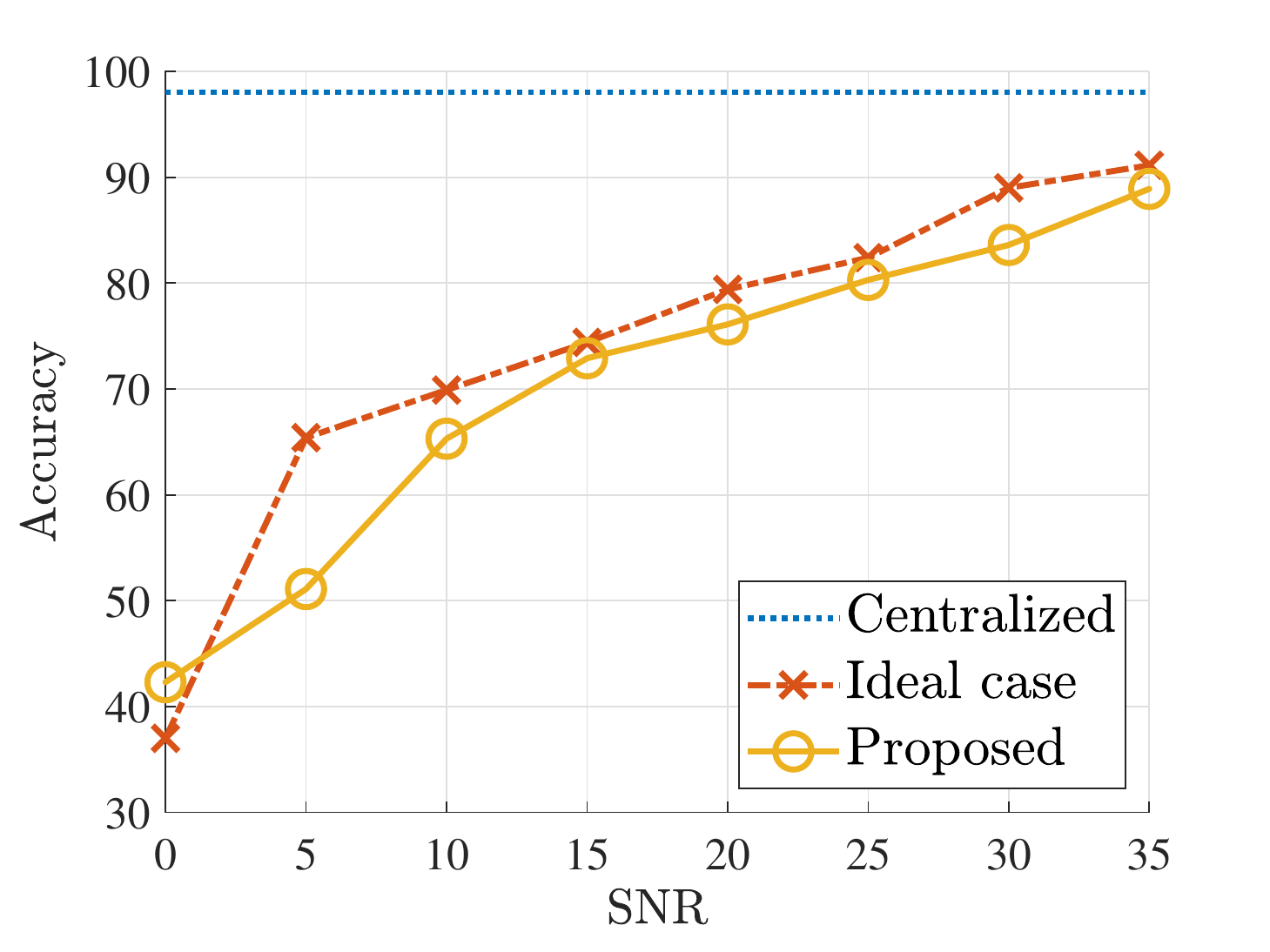}
	  \caption{The performance of different algorithms under stable massive MIMO channels.}\label{fig::massive}
\end{figure}
\subsection{Results Under Stable Massive MIMO Channels}\label{sec::massive}
We show the results under stable massive MIMO channels with $r=8$ and ResNet in Fig. \ref{fig::massive}, which does not show apparent discrimination compared to Fig. \ref{fig::sparse}, showing the effectiveness of our proposed system in massive MIMO scenarios. However, in Fig. \ref{fig::massive}, the proposed algorithm cannot outperform the ideal case, which is mainly caused by that the proposed algorithm brings much error when the dimension becomes larger.

\begin{figure}
    \centering
       \includegraphics[width=0.45\textwidth]{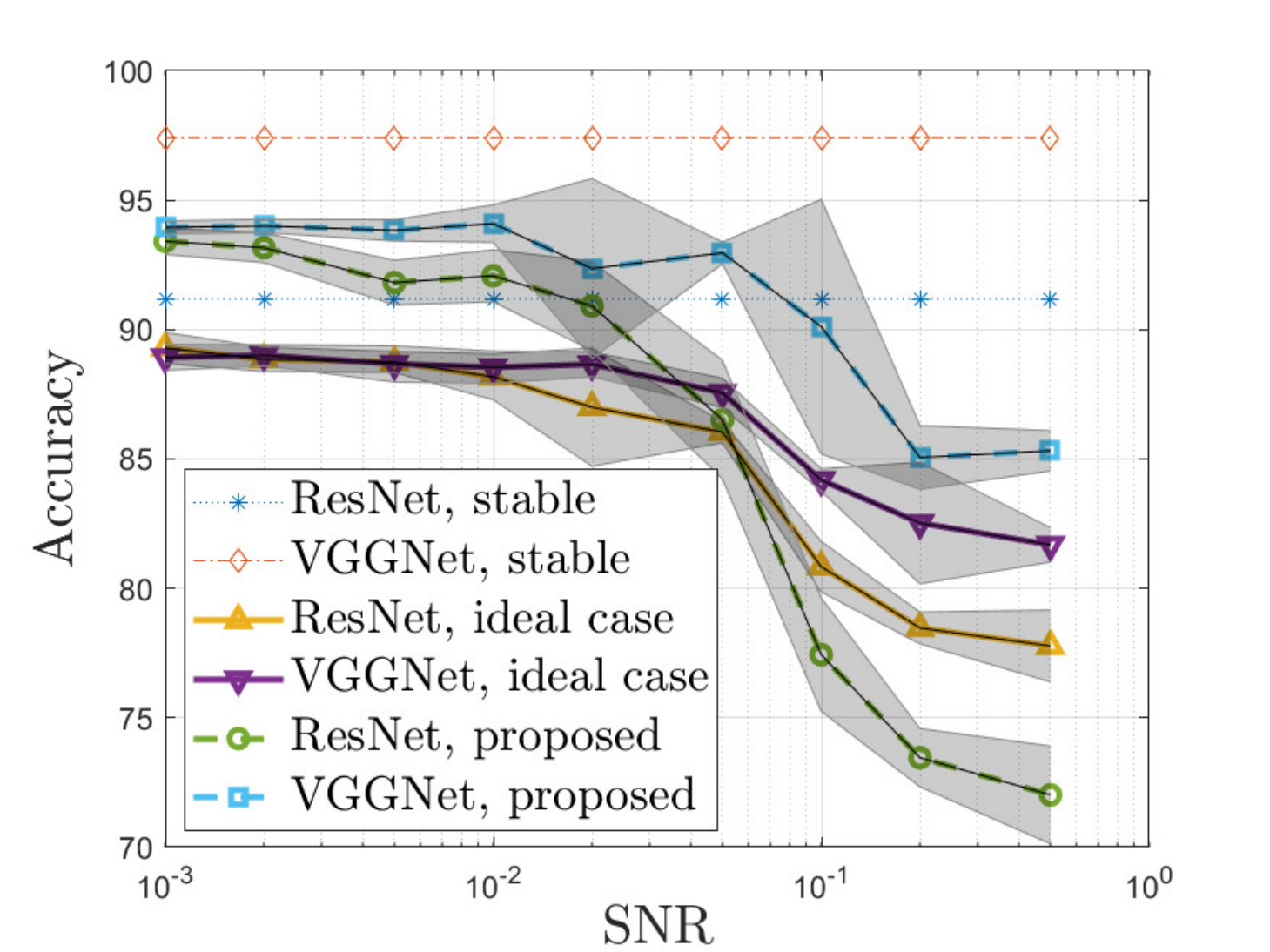}
	  \caption{The performance of different algorithms under moving MIMO channels, the results are averaged among 5 independent runs.} \label{fig::moving}
\end{figure}
\subsection{Results Under Time-Varying Channels}\label{sec::moving}
We show the results under dynamic channel settings here. In this setting, both the channel matrices vary according to the following memory model:
\begin{eqnarray}
a_{n,t+1}=(1-\rho)a_{n,t}+\rho \hat{a}_n,\\
\theta_{n,t+1}=(1-\rho)\theta_{n,t}+\rho\hat{\theta}_n,\\
\phi_{n,t+1}=(1-\rho)\phi_{n,t}+\rho\hat{\phi}_n,
\end{eqnarray}
where $\mathbf{a}_n$, $\theta_n$, and $\phi_n$ are the value of the corresponding values in \eqref{channel} at $t$-th time slot, respectively. $\hat{a}_n$, $\hat{\theta}_n$, and $\hat{\phi}_n$ are generated randomly following the way in \eqref{channel}, and $\rho$ indicates the changing rate. The channel $\mathbf{H}_t$ at each time slot is also generated by \eqref{channel}.
{In Fig. \ref{fig::moving}, we show the average values and standard deviations of the results of 5 independent runs. The standard deviations are shown by the shadows.}
We can find that both the ideal case and the proposed algorithm perform well even under moving channels.
In the ideal case, both the precoding and the combining matrices suddenly change when the channel changes, while the proposed algorithm can learn to adapt to the change of channel. Hence, the proposed algorithm outperforms the ideal case with smooth channel change, while performing poorly when the channel changes fiercely.
It is mainly because the channel is considered implicit parameters in the NN.
The training method is an iterative gradient algorithm, which cannot adapt the implicit parameters' change very well.
Moreover, the performance of VGGNet is still better since it is wider as we have explained above.

\section{Conclusions and Future Directions}\label{sec::6}
In this paper, we proposed a MIMO-based OAC system for implementing split ML over wireless networks, which significantly improves communication efficiency.
In the proposed system, the precoder and combiner design, together with the implicit MIMO channel matrix, contribute to a trainable layer, which can be natively integrated into a NN.
We theoretically found the inherit equivalent procedure between the reciprocity of MIMO channels and the training process of NNs.
Based on this finding, the explicit channel estimation process is eliminated in our system, which can improve the system's efficiency.
Besides, we provided basic principles for implementing the proposed system such that the proposed structure can be mathematically equivalent to any fully connected or convolutional layer in NNs.
Numerical results show that the proposed system has a similar converge rate with slight performance degradation under high-quality channels compared to the centralized algorithm.
We also observed that our algorithm could work well under various channels and is robust to slowly varying channels.
However, since iterative algorithms cannot adapt to fast-changing implicit parameters, the proposed system can only be applied to systems with quasi-stable channels.

The numerical results show that the NN structure may affect the performance of the proposed system, requiring a specific design of the NN structure such as that in \cite{AirNet}.
In this paper, we considered channel  reciprocity and the extension to natural channels without strict reciprocity remains to be a future direction.
The proposed MIMO OAC-based split ML structure can play an essential role in native distributed learning problems for wireless communications, such as distributed sensing, channel estimation, and prediction, location estimation and prediction, distributed network optimization, interference coordination, etc.

\begin{appendices}
\section{Proof of Theorem \ref{the::3}}\label{proof_the::3}
\begin{proof}
To prove Theorem \ref{the::3}, we introduce the regret to analyze the performance of the learning algorithm \cite{Zinkevich03}.
In particular, the regret is defined as
\begin{equation}
R(T)=\sum_{t=1}^{T} f_{t}\left(\bm{\theta}^{(t)}\right)-\sum_{t=1}^{T} f_{t}(\bm{\theta}^*),
\end{equation}
where $\bm{\theta}^*$ is the optimal parameter set.
Since function $f_t(\bm{\theta})$ is convex, we have
\begin{equation}
f_t(\bm{\theta}^{(t)})-f_t({\bm{\theta}^*})\leq\langle\bm{g}_t,\bm{\theta}^{(t)}-\bm{\theta}^*\rangle,
\end{equation}
which shows that
\begin{equation}
R(T)\leq\sum_{t=1}^T \langle\bm{g}_t,\bm{\theta}^{(t)}-\bm{\theta}^*\rangle.\label{regret}
\end{equation}

According to the update procedure of SGD with noise, we have
\begin{equation}
\bm{\theta}^{(t+1)}=\bm{\theta}^{(t)}-\eta_t(\bm{g}_t+\bm{\Lambda}_t \bm{n}_t),
\end{equation}
where $\bm{\Lambda}_t$ is a non-negative diagnose matrix, each element of which is no more than $\xi$, and $\bm{n}_t\sim\mathcal{N}(\bm{0}^{N*1},\sigma\bm{I}^{N*N})$ is the noise.
Hence, we have
\begin{equation}
\begin{aligned}
&\|\bm{\theta}^{(t+1)}-\bm{\theta}^*\|_2^2\\
=& \|\bm{\theta}^{(t)}-\bm{\theta}^*-\eta_t(\bm{g}_t+\bm{\Lambda_t} \bm{n}_t)\|_2^2\\
=& \|\bm{\theta}^{(t)}-\bm{\theta}^*\|_2^2-2\eta_t\langle\bm{g}_t, \bm{\theta}^{(t)}-\bm{\theta}^*\rangle+\eta_t^2\|\bm{g}_t\|_2^2\\
& +\eta_t^2|\bm{\Lambda}_t|^2\|\bm{n}_t\|_2^2-2\eta_t\langle\bm{\Lambda}_t \bm{n}_t, \bm{\theta}^{(t)}-\bm{\theta}^*\rangle+2\eta_t\langle\bm{\Lambda_t} \bm{n}_t,\bm{g}_t\rangle.
\end{aligned}\label{eq39}
\end{equation}
According to \eqref{eq39}, we can obtain
\begin{equation}
\begin{aligned}
\langle\bm{g}_t, \bm{\theta}^{(t)}-\bm{\theta}^*\rangle=&\frac{1}{2\eta_t}(\|\bm{\theta}^{(t)}-\bm{\theta}^*\|_2^2-\|\bm{\theta}^{(t+1)}-\bm{\theta}^*\|_2^2)\\
&+\frac{\eta_t}{2}|\bm{\Lambda}_t|^2\|\bm{n}_t\|_2^2+\frac{\eta_t}{2}\|\bm{g}_t\|_2^2\\
&-\langle\bm{\Lambda}_t \bm{n}_t, \bm{\theta}^{(t)}-\bm{\theta}^*\rangle+\eta_t\langle\bm{\Lambda_t\bm{n}_t,\bm{g}_t\rangle}\\
\leq & \frac{1}{2\eta_t}(\|\bm{\theta}^{(t)}-\bm{\theta}^*\|_2^2-\|\bm{\theta}^{(t+1)}-\bm{\theta}^*\|_2^2)\\
 &+\frac{\eta_t\xi}{2}\|\bm{n}_t\|_2^2+\frac{\eta_t G^2}{2} - \langle\bm{\Lambda}_t \bm{n}_t, \bm{\theta}^{(t+1)}-\bm{\theta}^*\rangle,\label{proof-1}
\end{aligned}
\end{equation}
where the last inequality comes from the fact that $\|\bm{g}_t\|_2^2\leq G$. Since there are always many parameters in a neural network, i.e., $N$ is very large, we can approximately get $\|\bm{n}_t\|_2^2\approx N\sigma^2$ based on the law of large numbers.
Substituting \eqref{proof-1} to \eqref{regret}, we have
\begin{equation}
\begin{aligned}
R(T)\leq & \sum_{t=1}^T\frac{1}{2\eta_t}(\|\bm{\theta}^{(t)}-\bm{\theta}^*\|_2^2-\|\bm{\theta}^{(t+1)}-\bm{\theta}^*\|_2^2)\\
& + \frac{\xi\|\bm{n}_t\|_2^2}{2}\sum_{t=1}^T \eta_t - \sum_{t=1}^T \langle\bm{\Lambda}_t \bm{n}_t, \bm{\theta}^{(t+1)}-\bm{\theta}^*\rangle.\label{regret2}
\end{aligned}
\end{equation}

Due to the fact that the parameters are bounded and learning rate is decreasing over time, we have
\begin{equation}
\begin{aligned}
&\sum_{t=1}^T\frac{1}{2\eta_t}(\|\bm{\theta}^{(t)}-\bm{\theta}^*\|_2^2-\|\bm{\theta}^{(t+1)}-\bm{\theta}^*\|_2^2)\\
=&  \frac{1}{2\eta_1}\|\bm{\theta}^{(1)}-\bm{\theta}^*\|_2^2 - \frac{1}{2\eta_T}\|\bm{\theta}^{(T+1)}-\bm{\theta}^*\|_2^2\\
&+\sum_{t=2}^T(\frac{1}{2\eta_t}-\frac{1}{2\eta_{t-1}})\|\bm{\theta}^{(1)}-\bm{\theta}^*\|_2^2\\
 \leq& \frac{1}{2\eta_1} D^2 + \sum_{t=2}^T(\frac{1}{2\eta_t}-\frac{1}{2\eta_{t-1}}) D^2\\
=&  \frac{D^2}{2\eta_T}.
\end{aligned}
\end{equation}
Since $\bm{n}_t$ follows a standard Gaussian distribution, we can obtain expectation $\mathbb{E}\{1/T\sum_{t=1}^T$ $ \langle\bm{\Lambda}_t \bm{n}_t, \bm{\theta}^{(t)}-\bm{\theta}^*\rangle\}=0$ and standard deviation $\sigma\{1/T\sum_{t=1}^T \langle\bm{\Lambda}_t \bm{n}_t, \bm{\theta}^{(t)}-\bm{\theta}^*\rangle\}=\mathcal{O}(\sigma  T^{-1/2})$.

When we choose $\eta_t=C/t^{1/2}$, we have
\begin{equation}
\begin{aligned}
R(T)/T\leq &\frac{D^2}{2C T^{1/2}}+\frac{C(\xi\|\bm{n}_t\|_2^2 + G^2)}{2T}\sum_{t=1}^T t^{-1/2}\\
&-\frac{1}{T}\sum_{t=1}^T \langle\bm{\Lambda}_t \bm{n}_t, \bm{\theta}^{(t)}-\bm{\theta}^*\rangle.\label{regret3}
\end{aligned}
\end{equation}
In the right hand of \eqref{regret3}, the first item is of $\mathcal{O}(T^{-1/2})$, the second item is approximately by $\mathcal{O}(\sigma^2 T^{-1/2})$, and the third item is a random item with zero expectation and standard deviation of $\mathcal{O}(\sigma T^{-1/2})$. Hence, the convergence rate of the algorithm is $\mathcal{O}(\sigma^2 T^{-1/2})$.
Comparing \eqref{regret3} with the regret of SGD provided in \cite{Zinkevich03} that
\begin{equation}
R(T)/T\leq \frac{D^2}{2C T^{1/2}}+\frac{C G^2}{T}\sum_{t=1}^T t^{-1/2},
\end{equation}
we can easily observe that noise power has a linear effect on the convergence rate of SGD with noise.
Theorem \ref{the::3} is proved.
\end{proof}
\end{appendices}
\printbibliography
\end{document}